\documentclass[10pt,twocolumn,letterpaper]{article}

\usepackage{wacv}              

\usepackage{graphicx}
\usepackage{amsmath}
\usepackage{amssymb}
\usepackage{booktabs}
\usepackage{multicol, multirow}
\usepackage{colortbl}
\usepackage{pgfplots}
\usepackage{tikz}
\definecolor{lightred}{RGB}{255, 204, 204}
\definecolor{lightblue}{RGB}{204, 229, 255}
\definecolor{lightgreen}{RGB}{204, 255, 204}
\definecolor{lightyellow}{RGB}{255, 255, 204}
\definecolor{lightgray}{gray}{0.9}
\definecolor{lightorange}{RGB}{255, 230, 153}
\usetikzlibrary{shapes.geometric}
\usepackage{pgfplotstable}
\usepgfplotslibrary{groupplots}
\usepackage{caption}
\usepackage{subcaption}
\usepackage{amsmath,amssymb,amsthm}

\newtheorem{lemma}{Lemma}[section]
\newtheorem{theorem}{Theorem}[section]

\usepackage[pagebackref,breaklinks,colorlinks]{hyperref}

\usepackage[capitalize]{cleveref}
\crefname{section}{Sec.}{Secs.}
\Crefname{section}{Section}{Sections}
\Crefname{table}{Table}{Tables}
\crefname{table}{Tab.}{Tabs.}

\def\wacvPaperID{1594} 
\def\confName{WACV}
\def\confYear{2025}

\begin{document}

\title{Forget Less by Learning Together through Concept Consolidation}

\author{
Arjun Ramesh Kaushik \qquad Naresh Kumar Devulapally \qquad Vishnu Suresh Lokhande \\ Nalini Ratha \qquad Venu Govindaraju\\
University at Buffalo, SUNY\\
{\tt\small \{kaushik3, devulapa, vishnulo, nratha, govind\}@buffalo.edu}
}
\maketitle

\begin{abstract}
Custom Diffusion Models (CDMs) have gained significant attention due to their remarkable ability to personalize generative processes. However, existing CDMs suffer from catastrophic forgetting when continuously learning new concepts. Most prior works attempt to mitigate this issue under the sequential learning setting with a fixed order of concept inflow and neglect inter-concept interactions. In this paper, we propose a novel framework - Forget Less by Learning Together (FL2T) - that enables concurrent and order-agnostic concept learning while addressing catastrophic forgetting. Specifically, we introduce a set-invariant inter-concept learning module where proxies guide feature selection across concepts, facilitating improved knowledge retention and transfer. By leveraging inter-concept guidance, our approach preserves old concepts while efficiently incorporating new ones. Extensive experiments, across three datasets, demonstrates that our method significantly improves concept retention and mitigates catastrophic forgetting, highlighting the effectiveness of inter-concept catalytic behavior in incremental concept learning of ten tasks with at least 2\% gain on average CLIP Image Alignment scores.
\end{abstract}

\begin{figure}
    \centering
    \includegraphics[width=0.5\textwidth]{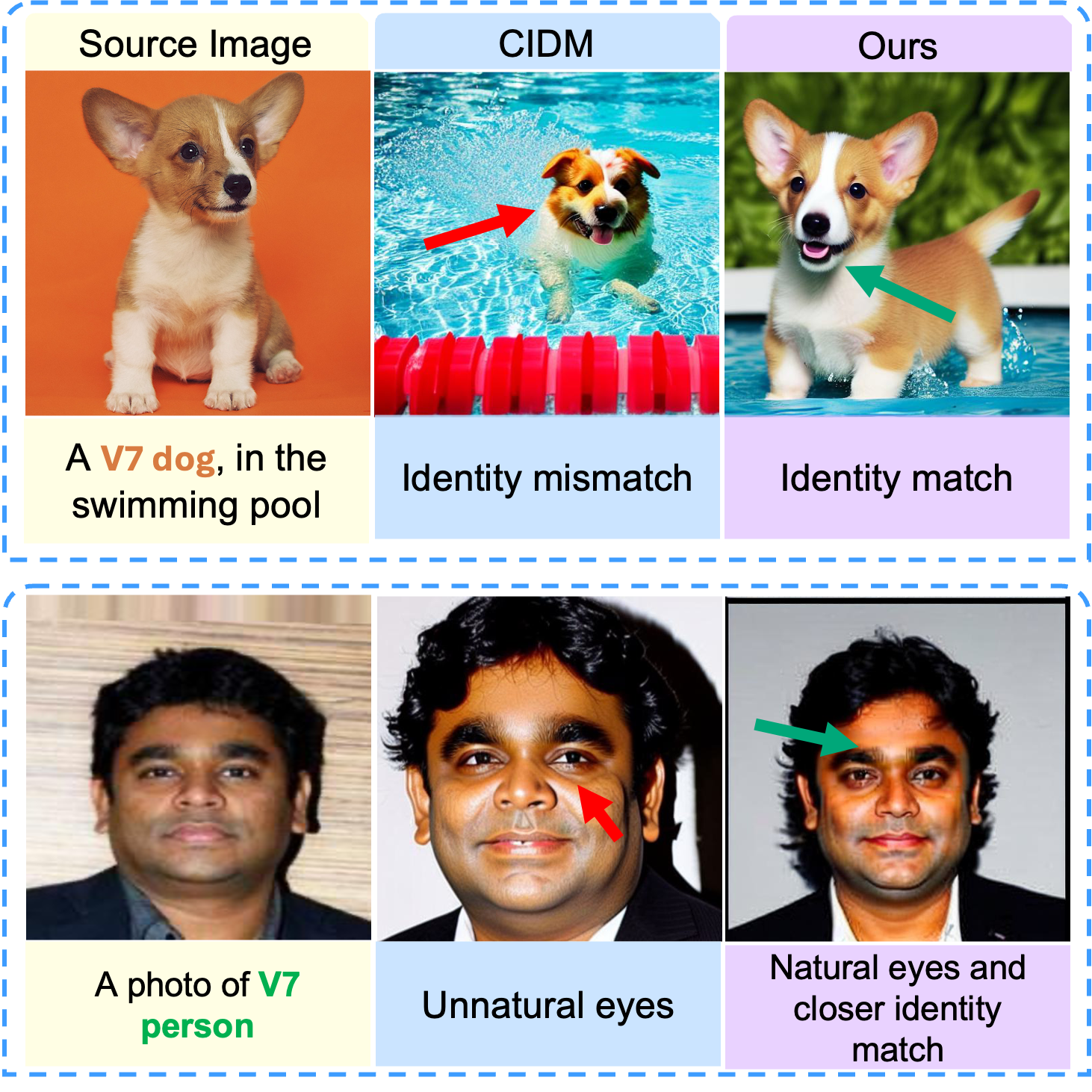}
    \caption{\textbf{Forget Less by Learning Together (FL2T).} Our method focuses on learning $G$ concepts in an order-agnostic incremental learning problem setting with fewer parameters and fewer reference images while also mitigating catastrophic forgetting. Unlike previous works, we leverage inter-concept interactions positively. The above image showcases examples of generated images, with a source image and an associated text prompt (left column) as input. We evaluate the generated image (right column) against the SOTA model CIDM \cite{cidm} (middle column). The red and green arrows indicate regions of undesirable and desirable qualities, and their reasons are stated below each image. }
    \label{fig:teaser}
\end{figure}

\section{Introduction}
\label{sec:intro}
\noindent Advancements in training, architectures, and datasets have enhanced text-conditioned generative models, enabling photorealistic text-to-image synthesis \cite{ldm_text1, ldm_text2, rifegan}. Conditional diffusion models facilitate high-fidelity image generation from text prompts and support user control over scene layout and sketches \cite{referring_image_editing, mokady2022nulltextinversioneditingreal, ranni, apisr}. This level of quality in text-conditioned generative models opens up new avenues for customization and personalization, allowing for the generation of images featuring specific objects or individuals in novel contexts or backgrounds. The ability to precisely capture concepts from reference images and seamlessly integrate them with text prompts is a significant advancement. Early approaches to model customization typically involved fine-tuning a subset of model parameters \cite{yu2024dreamsteererenhancingsourceimage} or text token embeddings \cite{cdm1} using a small set of user-provided reference images, often with additional regularization techniques \cite{wang2024multiclasstextualinversionsecretlyyields}. However, this fine-tuning process, which is required for each new concept, is both time-intensive and computationally costly. As a result, recent efforts have shifted toward developing tuning-free methods that eliminate the need for per-object optimization during inference, thus offering a more efficient and scalable solution for customization.

Concurrently, efforts are underway to expand the capability of model customization beyond a single, individual concept, aiming to seamlessly compose multiple new concepts together. However, this approach introduces additional challenges, such as the integration of unseen concepts and the potential for catastrophic forgetting of previously learned personalized concepts when new concepts are learned in succession through a concept-incremental framework. Additionally, issues like concept neglect may arise when performing multi-concept composition based on user-defined conditions \cite{cidm, mix_of_show}. A common setup in this area involves continuous customization with new, fine-grained concepts that have not been encountered before, typically with only a few examples provided for each new concept. This setup aligns with the concept of continual learning, which focuses on training a model across a sequence of tasks, each with a distinct data distribution, while striving to retain previously acquired knowledge without interference or forgetting.

The continual learning framework provides a promising direction for multi-concept personalization \cite{Kumari2023MultiConcept, Smith2024ContinualDiffusion, cidm}. Fundamentally, multi-concept personalization is closely related to training with larger datasets, as it inherently involves access to a greater number of reference images compared to single-concept personalization \cite{cidm, hao2023vicoplugandplayvisualcondition, dreambooth}. Existing approaches predominantly adopt a sequential concept learning paradigm, where new concepts are learned one after another. While this is practical, it raises the question: \textit{If reference images for all concepts are available from the outset, why not leverage them for simultaneous joint learning?} This paper aims to address this critical gap by proposing a framework that enables the concurrent learning of multiple concepts, allowing each concept to aid in the learning of others \cite{Kumari2023MultiConcept}. Beyond mitigating catastrophic forgetting, this joint learning approach fosters mutual knowledge transfer between concepts, enhancing overall learning efficiency and improving the quality of multi-concept personalization.

Most vision-language problems are instance-based, typically involving the mapping of a fixed-dimensional input tensor to a corresponding target value. However, multi-concept learning can be better formulated as a set-input problem \cite{lee2019settransformerframeworkattentionbased}, where the input consists of a set of concepts or reference images, and the goal is to generate a corresponding output for the entire set. Effectively addressing this requires treating multi-concept learning within the framework of set-based modeling. A model designed for set-input problems must satisfy two key properties: (1) \textbf{permutation invariance}, meaning that the model's output should remain unchanged regardless of the order in which input elements are presented, and (2) \textbf{scalability}, ensuring that the model can handle input sets of varying sizes without requiring retraining. Building upon these principles, this paper introduces a novel approach to multi-concept learning that jointly learns from multiple concepts while adhering to these fundamental properties.

Specifically, the contributions of this paper are as follows: 
\begin{enumerate}
    \item We propose a two-step joint / parallel learning framework - \textbf{Forget Less by Learning Together (FL2T)} - using a text-to-image diffusion model, where we offer flexibility in the order of concept inflow. 
    \item Utilizing the permutation invariance property of transformers, we develop a concept interaction module that uses proxy embeddings to help alleviate catastrophic forgetting. Additionally, FL2T also accounts for scalability, showing improved performance over SOTA methods with fewer parameters and fewer reference images.
\end{enumerate}

\section{Related Work}
\label{sec:relatedWork}


\noindent\textbf{Continual learning for generative models.}
For VAEs, Ye,  Bors \cite{YeBors2022OCM,YeBors2023CGKD} propose online cooperative memorization and continual generative knowledge distillation to mitigate forgetting. Within diffusion models, Zaj \etal \cite{Zajac2023ExploringCLDM} provide an early study that benchmarks Continual Learning baselines and reports timestep-dependent forgetting patterns and strong replay variants. Masip \etal \cite{Masip2024GenDistill} introduce \emph{generative distillation} tailored to diffusion, improving stability across tasks. Recent works \cite{Yoo2024LifelongVideo} also explore lifelong training settings beyond still images, e.g., lifelong video diffusion from a single stream using experience replay and dynamic architectural expansion for diffusion \cite{Ye2025DynamicExpansionDiffusion}.

\noindent\textbf{Forgetting in personalization/customization.}
Open-world customization emphasizes that adapting a model to new concepts can induce broad semantic and appearance drift, i.e., forgetting beyond the targeted concepts \cite{Laria2024OpenWorldForgetting}. In continual customization, \emph{Continual Diffusion (C-LoRA)} \cite{Smith2024ContinualDiffusion} adds concepts sequentially with parameter-efficient adapters and explicit self-regularization.

\noindent\textbf{Personalization and multi-concept customization.}
Foundational personalization methods such as Textual Inversion \cite{cdm1} and DreamBooth \cite{dreambooth} adapt large text-to-image models to user-provided subjects from a few sample images. For multi-concept learning and composition, Multi-Concept Customization \cite{Kumari2023MultiConcept} and Mix-of-Show \cite{Gu2023MixOfShow} introduce mechanisms for jointly learning and composing multiple concepts.

\noindent\textbf{Interpretability and intra-concept dynamics in diffusion.}
Analyses of cross-attention provide tools for understanding concept binding and interference during denoising. Prompt-to-Prompt shows that cross-attention controls spatial grounding and can be edited to localize changes \cite{Hertz2023P2P}, while DAAM \cite{Tang2023DAAM} performs token-level attribution in Stable Diffusion. These insights motivate designing mechanisms that preserve per-concept subspaces while supporting compositionality.

Prior work \cite{Zajac2023ExploringCLDM,Masip2024GenDistill} on diffusion CL highlights timestep-dependent forgetting and the utility of distillation/replay; continual customization targets sequential concept addition with adapters \cite{Smith2024ContinualDiffusion}. Our approach addresses this intersection by enabling \emph{order-agnostic} concept learning with mechanisms aimed at reducing cross-concept interference while maintaining compositional flexibility.

\section{Problem Definition}
\label{sec:problem}
\noindent In light of these challenges, we introduce a problem setting called \textbf{Order-agnostic Concept-Incremental Flexible Customization}. Building from CIFC \cite{cidm}, we consider a setting with a variable order of concept inflow. For our problem statement, we assume that the model learns from an undefined series of text-guided concept customization tasks $T = \{T_g\}_{g=1}^{G}$, where $G$ denotes the total number of tasks. Each task $T_g$ consists of a dataset $T_g = \{(x_g^k, p_g^k, y_g^k)\}_{k=1}^{n_g}$ where $n_g$ is the number of triplets in the task, $x_k^g$ is an image, $p_k^g$ is a text prompt (e.g., ``photo of a [V$^*$] [V$_{cat}$]''), and $y_k^g \in Y_g$ represents the concept tokens in $p_k^g$. Similar to CIFC, our setting also supports diverse customization tasks, including multi-concept generation \cite{multi_lora}, style transfer \cite{zhang2023inversionbasedstyletransferdiffusion}, and image editing \cite{chen2024anydoorzeroshotobjectlevelimage}. We define $Y_g$ as the concept space for task $g$, $Y_g = \bigcup_{k=1}^{n_g} y_k^g$ and $C_g$ represents the number of new concepts introduced in task $g$. In summary, our problem setting consists of three constraints -
\begin{enumerate}
    \item \textbf{All concepts are distinct :} $ Y_g \cap \left(\bigcup_{i=1}^{G-1} Y_i \right) = \emptyset$ indicating that new concepts in task $g$ are distinct from those learned in the other $G-1$ tasks. 
    \item \textbf{Concepts can be learned in any order :} $ T = \{ \pi(T_1 \cup T_2 \cup \dots \cup T_G) \mid \pi \in S_G \} $
indicating that \( T \) consists of any combination of the \( G \) concepts in any order, where \( S_G \) denotes the symmetric group of \( G \) elements, representing all possible permutations. 
    \item \textbf{No-replay constraint :} No memory storage is allocated to retain training data from past tasks, ensuring purely incremental learning of personalized concepts.
\end{enumerate}
 This setting allows models to continuously adapt to new personalization tasks while being exposed to catastrophic forgetting of prior concepts.

\section{Background}
\label{sec:prelim}

\noindent Despite their excellent generation capabilities, existing Custom Diffusion Models (CDMs) \cite{cdm1, cdm2, cdm3, mix_of_show} (Refer \ref{subsec:bg_cdms} for details on CDMs) assume that the number of personalized concepts remains fixed over time, which is unrealistic in real-world applications where users continuously introduce new concepts. Moreover, these models struggle with \textit{catastrophic forgetting} \cite{rebuffi2017icarlincrementalclassifierrepresentation} of previously learned concepts. Recent works \cite{ewc, lwf, Smith2024ContinualDiffusion, l2dm, cidm} have proposed various techniques to address the variability in the number of concepts and catastrophic forgetting. In this paper, we will focus on the SOTA framework CIDM \cite{cidm}.

\textbf{Concept Consolidation Loss (CCL).} Introduced in \cite{cidm}, the main purpose of CCL is to facilitate incremental concept learning for a defined order of concept inflow while mitigating catastrophic forgetting. CCL consists of two aspects: Task-Specific Knowledge (TSP) and Task-Shared Knowledge (TSH).

TSP is enhances the discriminative ability of the model towards concepts in an incremental learning setup. It uses an orthogonal subspace regularizer to constrain the LoRA weights of different customization tasks. Given the low-rank weight $\Delta\theta_g = \{\Delta \mathbf{W}_g^l\}_{l=1}^{L}$ for task $g$, where $l$ denotes the $l$-th transformer layer and $L$ is the number of layers, each weight matrix can be factorized as $\Delta \mathbf{W}_g^l = \mathbf{A}_g^l \mathbf{B}_g^l$, where $\mathbf{A}_g^l \in \mathbb{R}^{a \times r}$ represents a low-rank concept subspace and $\mathbf{B}_g^l \in \mathbb{R}^{r \times b}$ is its linear weighting matrix. To regulate orthogonality between subspaces of different tasks, TSP minimizes the constraint $\mathcal{R}_1$. On the other hand, TSH captures the semantic patterns that persist across different tasks. It introduces a layer-wise common subspace \( \mathbf{W}_*^l \) that is shared across tasks.
A learnable projection matrix \( \mathbf{H}_i^l \) encodes common semantic attributes, optimized by minimizing reconstruction loss.  
\begin{equation}
\mathcal{R}_1 = \sum_{i=1}^{g-1} \sum_{l=1}^{L} \mathbf{A}_i^l (\mathbf{A}_g^l)^\top
\end{equation}
\begin{equation}
    \mathcal{R}_2 = \sum_{i=1}^{g} \sum_{l=1}^{L} \left\| \Delta \mathbf{W}_i^l - \mathbf{H}_i^l \mathbf{W}_*^l \right\|_F^2
\end{equation}
EWA is the foundation of inferencing where previously trained task learners are preserved with a 0.25\% memory overhead compared to the SD-1.5 model \cite{finetune2}. The learned LoRA weights are aggregated based on the semantic relationship between stored concept token embedding ($\hat{e}$) and current task embedding ($\hat{e}$) in order to reduce forgetting. The aggregated low-rank weight \( \Delta \mathbf{\hat{W}}^l \) in the \( l \)-th transformer layer is formulated as:

\begin{equation}\label{eq:ewa}
\mathbf{\mathcal{M}} = \max (\mathbf{\hat{c}}^l \cdot (\mathbf{\hat{e}}^l)^\top) \quad 
\Delta \mathbf{\hat{W}}^l = \sum_{i=1}^{g} \Delta \mathbf{W}_i^l \cdot \psi(\mathbf{\mathcal{M}})_i
\end{equation}
where \( \max(\cdot) \) operates along the row axis, and $ \psi(\mathbf{\mathcal{M}}) = \mathbf{\mathcal{M}}^2 / \|\mathbf{\mathcal{M}}^2\|_F \in \mathbb{R}^{g}$ normalizes the semantic relations. Here, \( \psi(\mathbf{\mathcal{M}})_i \) denotes the \( i \)-th element of \( \psi(\mathbf{\mathcal{M}})\).

\textbf{Inference.} After training and applying Eq. \ref{eq:ewa} to aggregate the learned low-rank weights, we obtain the updated denoising UNet \( \epsilon_{\theta'_{\ast}}(\cdot) \) for inference, where \( \theta'_{\ast} = \theta_0 + \Delta \theta_{\ast} \) and \( \Delta \theta_{\ast} = \{\Delta \mathbf{\hat{W}}^l\}_{l=1}^{L} \). This formulation effectively consolidates the essential characteristics of all personalized concepts.

\begin{figure*}[!ht]
    \centering
    \includegraphics[width=\textwidth]{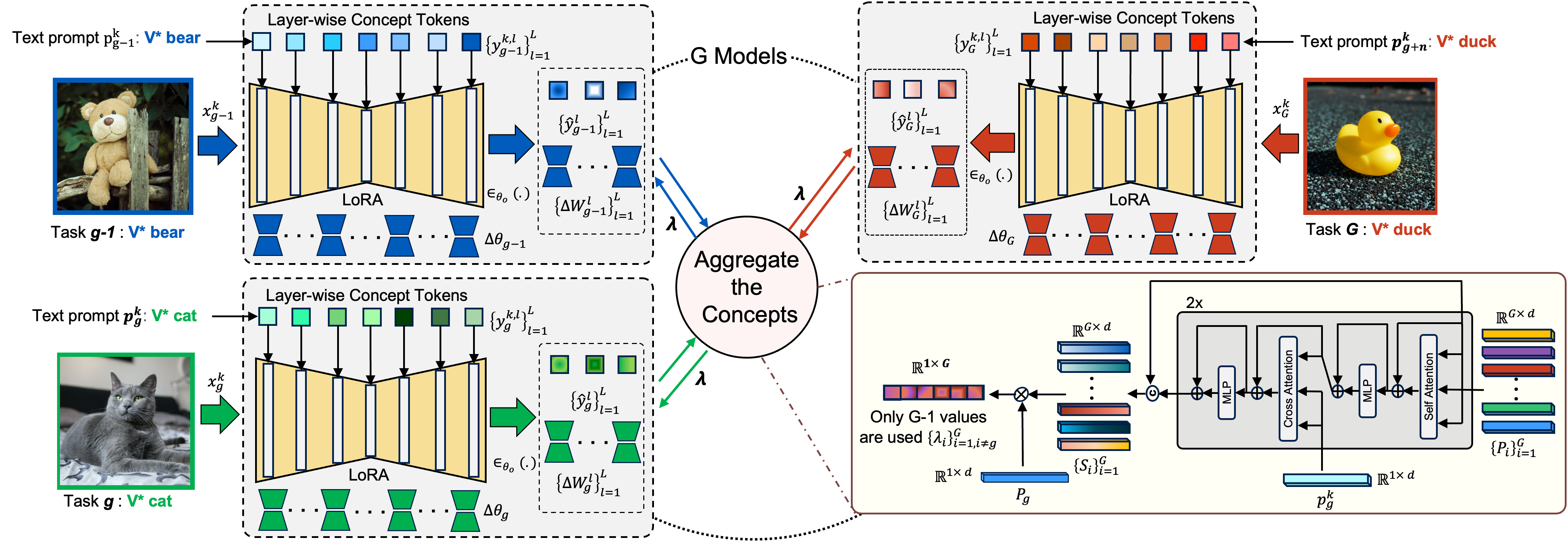}
    \caption{\textbf{FL2T overview.} We begin by independently training $G$ models on $G$ concepts for one epoch. Subsequently, in the second epoch, we utilize the learnt concept embeddings (from epoch one) across all concepts to perform cross-concept interactions or aggregate the concepts. To perform these interactions, we use the proxies (initialized with concept embeddings from epoch one) and two transformer layers. The goal of using transformer layers is to capture higher-order interactions between the concept embeddings, $g$-th concept embedding, and input prompt for task $g$. Subsequently, we compute a similarity matrix through matrix multiplication to weight the inter-concept interaction. This framework allows variable order of concept inflow, mitigates catastrophic forgetting, and outperforms SOTA frameworks with fewer parameters and fewer reference images.    }
    \label{fig:main}
\end{figure*}

\section{Forget Less by Learning Together (FL2T)}
\noindent We now present FL2T, a two-step training approach to personalizing diffusion models without the semantic loss of concepts. In the first step, we independently train $G$ different models, where $G$ is the number of concepts. Then, we consolidate all concepts into a single model, leveraging inter-concept interactions for effective personalization.

As in \cite{cidm}, to adapt the $\mathbf{g}$-th text-conditioned concept customization task $T_g$, we employ Low-Rank Adaptation (LoRA) \cite{lora1, lora2} to fine-tune the pretrained denoising UNet, commonly denoted using $\epsilon$. We define our pretrained UNet model, $\epsilon_{\theta_0}(\cdot)$, to operate on a set of personalized samples $\{\mathbf{x}_g^k, \mathbf{p}_g^k, \mathbf{y}_g^k\}_{k=1}^{n_g}$. Recall that $n_g$ is the number of triplets in the task, $x_k^g$ is an image, $p_k^g$ is a text prompt (e.g., ``photo of a [V$^*$] [V$_{cat}$]''), and $y_k^g \in Y_g$ represents the concept tokens in $p_k^g$. The model is trained on these samples to yield an updated model $\epsilon_{\theta'_g}(\cdot)$, where $\theta'_g$ is obtained by modifying the base parameters $\theta_0$ with task-specific updates, $\Delta\theta_g$: $\theta'_g = \theta_0 + \Delta\theta_g$. The task-specific update $\Delta\theta_g = \{\Delta \mathbf{W}_g^l\}_{l=1}^{L}$ consists of low-rank weight modifications for the $l$-th transformer layer. We know that, from Sec \ref{sec:prelim}, $\Delta \mathbf{W}_g^l \in \mathbb{R}^{a \times b}$ is factorized as $\mathbf{A}_g^l \mathbf{B}_g^l$, with $\mathbf{A}_g^l \in \mathbb{R}^{a \times r}$ and $\mathbf{B}_g^l \in \mathbb{R}^{r \times b}$. As demonstrated in prior research \cite{lora1, lora2}, with these adjustments, $\Delta\theta_g$ can encode the distinct characteristics of the personalized concept $\mathbf{C}_g$ for task $g$.

Solutions to our problem setting have been partly addressed in \cite{ewc, lwf, Smith2024ContinualDiffusion, l2dm, cidm}. Although these works reduce catastrophic forgetting, they all operate under a single-umbrella assumption that the order of concepts for customization tasks is fixed. Additionally, to the best of our knowledge, no work uses inter-concept interactions positively. To this end, we propose a novel framework that allows an order-agnostic incremental concept learning and uses inter-concept interactions positively to mitigate catastrophic forgetting.

In order to achieve an optimal order-agnostic continual learning framework, we leverage two intrinsic properties of attention: (1) \textbf{permutation invariance / set-invariance}, and (2) \textbf{ability to effectively capture higher-order interactions}. Set transformers, introduced in \cite{lee2019settransformerframeworkattentionbased}, have been used to encode pairwise or higher-order interactions between the elements of a set-invariant input. For an input set of $n$ d-dimensional observations, the Set Transformer combines attention over the $d$ input dimensions with nonlinear functions of pairwise interactions between the input observations. Further, we have also elucidated a detailed study describing the number of higher-order interactions captured by multiple-layers of attention in Sec. \ref{subsec:attn} and Table \ref{tab:interactions_compact}.

Additionally, a theoretical study on the drift of model weights showcases that FL2T (Ours) can achieve a lower model drift (see Sec. \ref{subsec:proof}).
We begin by analyzing one-step model drift when aggregating per-concept gradients $\{m_i\}_{i=1}^G$ either by uniform summation (CIDM: $M_{\mathrm{CIDM}}=\sum_i m_i$) (where $M$ denotes the model parameters) or by unnormalized attention weights (FL2T: $M_{\mathrm{FL2T}}=\sum_i \lambda_i m_i$ with $\lambda_i\in[-1,1]$).  
Lemma \ref{lem:upper-bound} establishes a universal upper bound: 
$\|M_{\mathrm{FL2T}}\|\;\le\;\sum_i\|\lambda_i\|$
which agrees with the worst-case drift of uniform summation and holds for any choice of $\lambda_i$ in $[-1,1]$.  

Then Theorem \ref{thm:existence-reduced} shows that whenever $M_{\mathrm{CIDM}}\neq 0$, one can choose some coefficients (not all equal to 1) so that $
\|M_{\mathrm{FL2T}}\| \;<\; \|M_{\mathrm{CIDM}}\|.
$ The constructive proof selects a concept $g$ whose gradient is positively aligned with the aggregate and slightly down-weights it by a small $\varepsilon>0$, reducing the norm of the update. Intuitively, unnormalized attention performs geometry-aware reweighting that allows cancellation of dominant components of $M_{\mathrm{CIDM}}$, thus consolidating concepts while producing smaller parameter drift per step. When $M_{\mathrm{CIDM}}=0$, drift is already minimal.  

This result formalizes why \textbf{attention-based aggregation} is a principled mechanism for our concept consolidation. It \textbf{attains lower update magnitude than uniform summation} while preserving flexibility to emphasize or de-emphasize individual concepts, mirroring ideas in multi-objective optimization about mitigating destructive interference between objectives.

\subsection{Independent Concept Training (Step 1)}
\label{subsec:step1}
 FL2T builds on the key empirical assumption that the concept inflow in Customized Diffusion Models (CDMs) is order-agnostic based on the problem setting in Sec \ref{sec:problem}. We begin by training each model $M_i$ on its respective task $T_i$ where $i \in [1,G]$. This helps each model learn a stable concept embedding $C_g$, capturing all semantics relevant to task $T_g$.

\subsection{Concept Aggregation (Step 2)} 

Recent works \cite{ewc, lwf, Smith2024ContinualDiffusion, l2dm, cidm} have all proposed methods to reduce inter-concept interaction, assuming that the interaction is harmful and can cause the overwriting of concepts. We hypothesise that this inter-concept interaction can rather catalyse generation abilities of concepts and reduce catastrophic forgetting. 

Consider the set of concepts to be a directed graph where an edge from concept A to B symbolizes relevant information for B contained within A and vice versa. We define our goal as learning this relevance weight for set-invariant concepts. There are two challenges associated with achieving this, (1) the relevance weights change based on the input prompt, and (2) input data of other concepts is inaccessible. To this end, we utilize learnable embeddings to capture the relevance weights between concepts. These learnable embeddings \( \{\mathbf{P}_i\}_{i=1}^G \), also known as proxy embeddings, are initialized with stable concept embeddings \( \{\mathbf{C}_i\}_{i=1}^G \) learned in Step 1. 

Motivated by the properties of attention in providing permutation invariance, higher-order interactions and lower model-drift, we propose to use transformer decoders for our order-agnostic (set-invariant) framework. The transformer decoder module intends to capture the relevance weights (higher-order interactions) between concepts through proxy embeddings. 

Transformer decoders \cite{vaswani2023attentionneed} consist of two main blocks: self-attention and cross-attention. The proxy embeddings \( \{\mathbf{P}_i\}_{i=1}^G \) serve as input to the decoder that first performs self-attention. Subsequently, it is cross-attended with input text-prompt for task $g$, outputtting \( \{\mathbf{P'}_i\}_{i=1}^G \). The objective of this is to capture the relevance scores between concepts for a particular text prompt. Although transformers are adept at capturing higher-order interactions, they are also susceptible to rank collapse \cite{rank_collapse}.

To address this rank collapse, we introduce two things: (1) a non-linear transformation layer such as an MLP, and  (2) contrastive loss. The attended proxy embeddings \( \{\mathbf{P'}_i\}_{i=1}^G \) are combined with the original concept embeddings: $S_i = \{f(\mathbf{C}_i | \mathbf{P'}_i)\}_{\forall i \neq g}$, where \( f(\cdot | \cdot) \) represents a non-linear transformation such as an MLP on two concatenated vectors. On the other hand, the contrastive loss constrains embeddings to be distinct and away from each other in terms of cosine distance while ensuring they capture relevant information. It is applied on $\{\mathbf{S}_i\}_{i=1}^G$ as given in Eq. \ref{eq:contra} where $\text{sim}(\mathbf{z}_i, \mathbf{z}_j) = \frac{\mathbf{z}_i \cdot \mathbf{z}_j}{\|\mathbf{z}_i\| \|\mathbf{z}_j\|}$ and $\mathbf{\tau}$ is the temperature hyperparameter. 

\begin{equation}\label{eq:contra}
\mathcal{R}_3 = \frac{1}{G} \sum_{i=1}^{G} -\log \frac{\exp(\text{sim}(S_i, S_i) / \tau)}{\sum_{j=1}^{N} \exp(\text{sim}(S_i, S_j) / \tau)}
\end{equation}

 Recall that in Sec \ref{sec:prelim}, we introduced Task-Specific Knowledge that regulates orthogonality between subspaces of different tasks. Here, we extend that by weighting the constraint with relevance weights between concepts depending on the input prompt $\mathbf{p}_g^k$ :  
\begin{equation}
\mathcal{R'}_1 = \sum_{i=1, i \neq g}^{G} \sum_{l=1}^{L} \lambda_{i} \mathbf{A}_i^l (\mathbf{A}_g^l)^\top 
\end{equation}

where $\lambda$ is the edge weight derived through proxy embeddings for each concept $g$. For a given task \( T_g \), we compute the importance weights \( \{\mathbf{\lambda}_i\}_{i=1, i \neq g}^{G} \) by measuring the similarity between the concept embedding \( \mathbf{C}_g \) and the transformed representations \( \mathbf{S}_i \) that has captured the higher-order interactions of concepts:
$\mathbf{\lambda}_i = \{\mathbf{C}_g \mathbf{S}_i^\top\}_{\forall i \neq g}
$. This formulation enables the selection of the most relevant concepts for each task.

Finally, from Sec \ref{sec:prelim}, we define the loss function as:

\begin{equation}\label{eq:cil}
\begin{split}
\mathcal{L} = \mathcal{E}_{z \sim E(x_g^k), c_g^k, \epsilon \sim \mathcal{N}(0, I), t} 
\Bigg[ \| \epsilon - \epsilon_{\theta'_g} (z_t | c_g^k, t) \|_2^2  \\
+ \mathcal{R'}_1 + \gamma_1 \mathcal{R}_2 + \gamma_2 \mathcal{R}_3 \Bigg]
\end{split}
\end{equation}

where \( \mathbf{c}_g^k = \{\mathbf{c}_g^{k,l}\}_{l=1}^{L} \) represents layer-wise textual embeddings, and \(\gamma\) is the trade-off hyperparameter.

\begin{table*}[!h]
    \centering
    \renewcommand{\arraystretch}{1.5}
    \resizebox{\textwidth}{!}{
    \begin{tabular}{cc|cc|cc|cc|cc|cc|cc|cc|cc|cc|cc|cccc}
        \hline
        \hline
        & \multirow{2}{*}{\textbf{Methods}} & \multicolumn{2}{|c|}{\textbf{V1}} & \multicolumn{2}{|c|}{\textbf{V2}} & \multicolumn{2}{|c|}{\textbf{V3}} & \multicolumn{2}{|c|}{\textbf{V4}} & \multicolumn{2}{|c|}{\textbf{V5}} & \multicolumn{2}{|c|}{\textbf{V6}} & \multicolumn{2}{|c|}{\textbf{V7}} & \multicolumn{2}{|c|}{\textbf{V8}} & \multicolumn{2}{|c|}{\textbf{V9}} & \multicolumn{2}{|c|}{\textbf{V10}} & \multicolumn{2}{|c}{\textbf{Avg.}} \\
        \cline{3-24}  
         & & \textbf{IA} & \cellcolor{lightblue}\textbf{TA} & IA & \cellcolor{lightblue}TA & 
         \textbf{IA} & \cellcolor{lightblue}\textbf{TA} & 
         \textbf{IA} & \cellcolor{lightblue}\textbf{TA} & 
         \textbf{IA} & \cellcolor{lightblue}\textbf{TA} & 
         \textbf{IA} & \cellcolor{lightblue}\textbf{TA} & 
         \textbf{IA} & \cellcolor{lightblue}\textbf{TA} & 
         \textbf{IA} & \cellcolor{lightblue}\textbf{TA} & 
         \textbf{IA} & \cellcolor{lightblue}\textbf{TA} & 
         \textbf{IA} & \cellcolor{lightblue}\textbf{TA} & 
         \textbf{IA} & \cellcolor{lightblue}\textbf{TA} \\
         \hline
         
         \cellcolor{lightyellow} & \multicolumn{23}{c}{\textbf{Previous frameworks}}\\
         \cellcolor{lightyellow} & \textbf{Finetuning} & 
         77.6 & \cellcolor{lightblue}64.4 & 
         82.2 & \cellcolor{lightblue}74.6 & 
         79.0 & \cellcolor{lightblue}69.4 & 
         77.6 & \cellcolor{lightblue}68.6 & 
         79.6 & \cellcolor{lightblue}75.0 & 
         62.9 & \cellcolor{lightblue}70.0 & 
         71.5 & \cellcolor{lightblue}\underline{76.7} & 
         53.7 & \cellcolor{lightblue}69.2 & 
         81.4 & \cellcolor{lightblue}65.4 & 
         72.1 & \cellcolor{lightblue}67.2 & 
         73.7 & \cellcolor{lightblue}70.0 \\
         
         \cellcolor{lightyellow} & \textbf{EWC} \cite{ewc} & 
         78.7 & \cellcolor{lightblue}67.1 & 
         83.8 & \cellcolor{lightblue}77.5 & 
         80.4 & \cellcolor{lightblue}72.7 & 
         80.3 & \cellcolor{lightblue}77.9 & 
         80.7 & \cellcolor{lightblue}76.7 & 
         64.0 & \cellcolor{lightblue}72.3 & 
         76.5 & \cellcolor{lightblue}74.2 & 
         57.1 & \cellcolor{lightblue}72.0 & 
         \textbf{84.4} & \cellcolor{lightblue}66.0 & 
         73.1 & \cellcolor{lightblue}70.4 & 
         75.9 & \cellcolor{lightblue}72.7 \\
         
         \cellcolor{lightyellow} & \textbf{LWF} \cite{lwf} & 
         80.4 & \cellcolor{lightblue}70.8 & 
         79.7 & \cellcolor{lightblue}75.2 & 
         80.9 & \cellcolor{lightblue}71.0 &  
         77.4 & \cellcolor{lightblue}77.4 & 
         80.9 & \cellcolor{lightblue}76.0 & 
         61.8 & \cellcolor{lightblue}71.7 & 
         73.2 & \cellcolor{lightblue}76.3 & 
         53.5 & \cellcolor{lightblue}72.9 & 
         78.1 & \cellcolor{lightblue}72.5 &
         74.7 & \cellcolor{lightblue}70.0 &
         74.1 & \cellcolor{lightblue}73.4 \\
         
         \cellcolor{lightyellow} & \textbf{CLoRA} \cite{Smith2024ContinualDiffusion} & 
         83.2 & \cellcolor{lightblue}69.4 & 
         83.4 & \cellcolor{lightblue}78.0 & 
         81.1 & \cellcolor{lightblue}74.1 & 
         80.6 & \cellcolor{lightblue}78.8 & 
         84.9 & \cellcolor{lightblue}76.4 & 
         66.3 & \cellcolor{lightblue}69.6 & 
         76.2 & \cellcolor{lightblue}\underline{76.7} & 
         58.1 & \cellcolor{lightblue}73.9 & 
         83.0 & \cellcolor{lightblue}69.0 & 
         72.1 & \cellcolor{lightblue}71.8 & 
         76.9 & \cellcolor{lightblue}73.6 \\
         
         \cellcolor{lightyellow} & \textbf{L2DM} \cite{l2dm} & 
         78.7 & \cellcolor{lightblue}68.6 & 
         86.3 & \cellcolor{lightblue}79.5 & 
         76.6 & \cellcolor{lightblue}70.1 & 
         80.7 & \cellcolor{lightblue}73.0 & 
         86.8 & \cellcolor{lightblue}76.7 & 
         70.8 & \cellcolor{lightblue}67.7 & 
         70.0 & \cellcolor{lightblue}75.9 & 
         59.3 & \cellcolor{lightblue}74.1 &
         77.7 & \cellcolor{lightblue}71.8 & 
         74.1 & \cellcolor{lightblue}69.4 & 
         76.1 & \cellcolor{lightblue}72.7\\
         \cline{2-24}
         \cellcolor{lightyellow} & \multicolumn{23}{c}{\textbf{On the effectiveness of concept aggregation}}\\

         \cellcolor{lightyellow} & \shortstack{\textbf{Ours} w/o \\ guidance}  & 
          82.4 & \cellcolor{lightblue}\underline{76.1} & 
          83.9& \cellcolor{lightblue}\underline{82.5} & 
          80.9& \cellcolor{lightblue}74.9 & 
          78.8& \cellcolor{lightblue}\underline{82.5} & 
          84.6& \cellcolor{lightblue}\underline{79.8} & 
          66.0& \cellcolor{lightblue}\underline{74.2} & 
          82.9& \cellcolor{lightblue}73.7 & 
          59.8& \cellcolor{lightblue}76.3 &
          81.8& \cellcolor{lightblue}73.4 & 
          75.6& \cellcolor{lightblue}71.6 & 
          77.7& \cellcolor{lightblue}\underline{76.5}\\

          \cellcolor{lightyellow} & \shortstack{\textbf{Ours} with \\ Cosine-guidance}  & 
          81.8 & \cellcolor{lightblue}74.8 & 
          84.6 & \cellcolor{lightblue}82.2 & 
          80.1 & \cellcolor{lightblue}\underline{75.9} & 
          77.9 & \cellcolor{lightblue}81.3 & 
          85.2 & \cellcolor{lightblue}79.2 & 
          \textbf{76.7} & \cellcolor{lightblue}72.0 & 
          83.9 & \cellcolor{lightblue}72.3 & 
          59.3 & \cellcolor{lightblue}77.1 &
          82.2 & \cellcolor{lightblue}73.4 & 
          \textbf{76.7} & \cellcolor{lightblue}\underline{72.2} & 
          78.8 & \cellcolor{lightblue}76.0\\
         \cline{2-24}
         
         \cellcolor{lightyellow} & \multicolumn{23}{c}{\textbf{Comparison against SOTA framework}}\\
         \cellcolor{lightyellow} & \textbf{CIDM}\cite{cidm} & 
         83.6 & \cellcolor{lightblue}75.3 & 
         \textbf{86.4} & \cellcolor{lightblue}78.1 & 
         82.9 & \cellcolor{lightblue}74.0 & 
         80.8 & \cellcolor{lightblue}81.1 & 
         86.5 & \cellcolor{lightblue}78.2 & 
         69.5 & \cellcolor{lightblue}70.1 & 
         73.7 & \cellcolor{lightblue}74.7 & 
         56.9 & \cellcolor{lightblue}74.3 & 
         82.4 & \cellcolor{lightblue}73.5 & 
         75.9 & \cellcolor{lightblue}70.2 & 
         78.0 & \cellcolor{lightblue}74.8 \\
         
         \cellcolor{lightyellow} & \shortstack{\textbf{Ours} with \\ proxy-guidance} &
         \textbf{84.4} & \cellcolor{lightblue}74.0 & 
         86.1 & \cellcolor{lightblue}81.6 & 
         \textbf{84.4} & \cellcolor{lightblue}70.1 & 
         \textbf{82.2} & \cellcolor{lightblue}81.1 & 
         \textbf{87.2} & \cellcolor{lightblue}78.0 & 
         69.3 & \cellcolor{lightblue}72.9 & 
         \textbf{85.3} & \cellcolor{lightblue}70.4 & 
         \textbf{60.7} & \cellcolor{lightblue}\underline{77.6} & 
         82.1 & \cellcolor{lightblue}\underline{76.6} & 
         76.5 & \cellcolor{lightblue}71.6 & 
         \textbf{79.8} & \cellcolor{lightblue}\underline{75.4} \\
        
        \rowcolor{lightgray} \multirow{-13}{*}{\cellcolor{lightyellow}\rotatebox{90}{\textbf{CIFC \cite{cidm}}}}  &
         \textbf{$\Delta$} & +0.8 & -1.3 & 
          -0.3 & +3.5 & 
          +1.5 & -4.0 & 
          +1.4 & +0.0 & 
          +0.7 & -0.2 &
          -1.5 & +0.6 & 
          +9.2 & -6.3 & 
          +1.4 & +3.3 & 
          -2.3 & +3.1 & 
          +0.6 & -0.2 & 
          +1.8 & +0.6 \\
         \hline
         \hline
         
          \cellcolor{lightgreen} & \multicolumn{23}{c}{\textbf{Comparison against SOTA framework}}\\
          
         \cellcolor{lightgreen} & \textbf{CIDM}\cite{cidm} & 
          74.2& \cellcolor{lightblue}\underline{59.0} & 
          67.6& \cellcolor{lightblue}60.4 & 
          73.0& \cellcolor{lightblue}59.4 & 
          75.4& \cellcolor{lightblue}58.2 & 
          70.1& \cellcolor{lightblue}58.0 & 
          78.1& \cellcolor{lightblue}59.0 & 
          81.0& \cellcolor{lightblue}58.4 & 
          70.0& \cellcolor{lightblue}58.0 & 
          70.1& \cellcolor{lightblue}62.1 & 
          73.0& \cellcolor{lightblue}55.4 & 
          73.3& \cellcolor{lightblue}58.8 \\
         
         \cellcolor{lightgreen} & \shortstack{\textbf{Ours} with \\ proxy-guidance} &
         \textbf{76.4} & \cellcolor{lightblue}58.9 & 
         \textbf{72.7} & \cellcolor{lightblue}\underline{60.7} & 
         \textbf{73.1} & \cellcolor{lightblue}\underline{60.0} & 
         \textbf{77.2} & \cellcolor{lightblue}\underline{63.0} & 
         \textbf{78.1} & \cellcolor{lightblue}\underline{61.5} & 
         \textbf{80.0} & \cellcolor{lightblue}\underline{60.9} & 
         \textbf{83.4} & \cellcolor{lightblue}\underline{60.6} & 
         \textbf{72.0} & \cellcolor{lightblue}\underline{63.1} & 
         \textbf{73.1} & \cellcolor{lightblue}\underline{62.7} & 
         \textbf{74.5} & \cellcolor{lightblue}\underline{57.9} & 
         \textbf{76.1} & \cellcolor{lightblue}\underline{60.9} \\
        
        \rowcolor{lightgray} \multirow{-4}{*}{\cellcolor{lightgreen}\rotatebox{90}{\textbf{CelebA}\cite{celeba}}}  &
         \textbf{$\Delta$} & 
         +2.2 & -0.1 & 
         +5.1 & +0.3 & 
         +0.1 & +0.6 & 
         +1.8 & +4.8 & 
         +8.0 & +3.5 &
         +1.9 & +1.9 & 
         +2.4 & +2.2 & 
         +2.0 & +5.1 & 
         +3.0 & +0.6 & 
         +1.5 & +2.5 & 
         +2.8 & +2.1\\
         \hline
         \hline

         \cellcolor{lightorange} & \multicolumn{23}{c}{\textbf{Comparison against SOTA framework}}\\
          
         \cellcolor{lightorange} & \textbf{CIDM}\cite{cidm} & 
          80.8& \cellcolor{lightblue}74.1 & 
          71.4& \cellcolor{lightblue}\underline{77.9} & 
          81.3& \cellcolor{lightblue}\underline{81.7} & 
          83.0& \cellcolor{lightblue}82.6 & 
          79.2& \cellcolor{lightblue}\underline{83.7} & 
          84.5& \cellcolor{lightblue}69.6 & 
          76.7& \cellcolor{lightblue}72.3 & 
          87.5& \cellcolor{lightblue}73.7 & 
          82.2& \cellcolor{lightblue}79.9 & 
          85.0& \cellcolor{lightblue}82.8 & 
          81.2& \cellcolor{lightblue}78.5 \\
         
         \cellcolor{lightorange} & \shortstack{\textbf{Ours} with \\ proxy-guidance} &
         \textbf{83.3} & \cellcolor{lightblue}\underline{74.8} & 
         \textbf{75.7} & \cellcolor{lightblue}73.1 & 
         \textbf{82.1}& \cellcolor{lightblue}79.8 & 
         \textbf{84.4}& \cellcolor{lightblue}\underline{83.3} & 
         \textbf{83.3}& \cellcolor{lightblue}\underline{83.7} & 
         \textbf{84.8}& \cellcolor{lightblue}\underline{70.2} & 
         \textbf{78.7}& \cellcolor{lightblue}\underline{73.9} & 
         \textbf{88.1}& \cellcolor{lightblue}\underline{77.5} & 
         \textbf{83.3}& \cellcolor{lightblue}\underline{80.2} & 
         \textbf{92.6}& \cellcolor{lightblue}\underline{86.8} & 
         \textbf{83.7} & \cellcolor{lightblue}\underline{79.0} \\
        
        \rowcolor{lightgray} \multirow{-4}{*}{\cellcolor{lightorange}\rotatebox{90}{\textbf{ImageNet}\cite{imagenet}}}  &
         \textbf{$\Delta$} & 
         +2.5 & +0.7 & 
         +4.3 & -4.8 & 
         +0.8 & -1.9 & 
         +1.4 & +0.7 & 
         +4.1 & +0.0 &
         +0.3 & +0.6 & 
         +2.0 & +1.6 & 
         +0.6 & +3.8 & 
         +1.1 & +0.3 & 
         +7.6 & +4.0 & 
         +2.5 & +0.5 \\
         \hline
         \hline
    \end{tabular}
    }
    \caption{\textbf{Quantitative Analysis.} We compare our method, denoted by ``Ours with proxy-guidance" against SOTA frameworks, observing an improvement in both IA and TA scores in at least 6 out of 10 concepts. Further experimentation with uncontrolled (``Ours w/o guidance") and cosine-similarity (``Ours with Cosine-guidance") based inter-concept interactions showcase the need for proxy-guidance on the interactions. IA and TA refer to CLIP \cite{clip} Image Alignment and Text Alignment scores respectively.}
    \label{tab:main_exp}
\end{table*}

\section{Experiments}
\noindent We train FL2T on three datasets - CIFC \cite{cidm}, CelebA \cite{celeba} and ImageNet \cite{imagenet}, each consisting of ten distinct concepts with 3-5 text-image pairs. We follow the evaluation strategy defined in \cite{cidm}. Specifically, twenty evaluation prompts are introduced for each of the ten concepts, and fifty images are generated per prompt. More information on datasets, hyperparameters and evaluation metrics has been provided in the Appendix.


\subsection{Ablation Studies}
This section analyzes the effect of our proxy-guided concept aggregation module on the variation in the number of LoRA parameters, number of transformer layers, and number of reference images.

\textbf{On the number of reference images.}
We further compare CIDM \cite{cidm} and FL2T (Ours) by imposing a constraint on the number of reference images while evaluating both models on the same prompts as the original dataset, as shown in Fig. \ref{fig:ablation_ref_images_lora}(a). As expected, the performance of both models improves linearly with an increasing number of reference images. However, FL2T consistently outperforms CIDM across all cases and achieves performance comparable to CIDM on the original dataset (\(\geq4\) images) with only three reference images, highlighting the efficiency and effectiveness of our approach. \textit{FL2T performs better than SOTA models with fewer reference images.}

\begin{figure}[!ht]
\centering
\begin{tikzpicture}
\begin{groupplot}[
  group style={
    group size=2 by 1,
    horizontal sep=0.8cm,
  },
  width=0.28\textwidth,
  height=0.26\textwidth,
  xmin=1, xmax=7,
  ymin=70, ymax=81,
  grid=both,
  grid style={line width=.1pt, draw=gray!10},
  major grid style={line width=.2pt,draw=gray!50},
  minor tick num=4,
  tick label style={font=\scriptsize},
]

\nextgroupplot[xlabel={\shortstack{(a) Number of \\ reference images}}, xmin = 1, xmax = 4, ymin =72, ymax =81,legend to name=zelda, legend columns =4, legend style={
font = \scriptsize,
    anchor=west
  }, font=\bfseries]
\addplot[mark=*, blue, dashed, line width=1pt] table[x=Ref, y=IA] {
Ref IA
1 73.2
2 74.6
3 76.4
4 78.0 
};
\addlegendentry{CIDM (IA)}
\addplot[mark=square*, red, dashed, line width=1pt] table[x=Ref, y=TA] {
Ref TA
1 77.0
2 76.7
3 75.9
4 74.8
};
\addlegendentry{CIDM (TA)}

\addplot[mark=triangle*, orange, line width=1.5pt] table[x=Ref, y=IA] {
Ref IA
1 76.3
2 76.3
3 78.0
4 79.8 
};
\addlegendentry{FL2T (IA)}
\addplot[mark=diamond*, violet, line width=1.5pt] table[x=Ref, y=TA] {
Ref TA
1 76.4
2 76.3
3 76.4
4 75.4
};
\addlegendentry{FL2T (TA)}



\nextgroupplot[xlabel={(b) LoRA rank}, ymin = 74, ymax = 81, font=\bfseries, xmin=3, xmax=6]
\addplot[mark=*, blue, dashed, line width=1.5pt, ] table[x=Ref, y=IA] {
Ref IA
3 77.1
4 78.0
5 77.8
6 77.9
};
\addplot[mark=square*, red, dashed, line width=1.5pt] table[x=Ref, y=TA] {
Ref TA
3 75.3
4 74.8
5 75.4
6 75.6
};

\addplot[mark=triangle*, orange, line width=1.5pt] table[x=Ref, y=IA] {
Ref IA
3 79.1
4 79.8
5 80.1
6 78.7
};
\addplot[mark=diamond*, violet, line width=1.5pt] table[x=Ref, y=TA] {
Ref TA
3 75.6
4 75.4
5 75.5
6 76.0
};

\end{groupplot}
\path (group c1r1.south west) -- (group c2r1.south east) coordinate[midway] (centerbelow);
\node[anchor=north] at (centerbelow |- current bounding box.south) {\pgfplotslegendfromname{zelda}};
\end{tikzpicture}
\caption{\textbf{Ablation studies.} We compare the performance of FL2T against CIDM: with a constraint on the (a) number of reference images, and (b) on the LoRA rank. In each case, FL2T exhibits superior performance.}
\label{fig:ablation_ref_images_lora}
\end{figure}
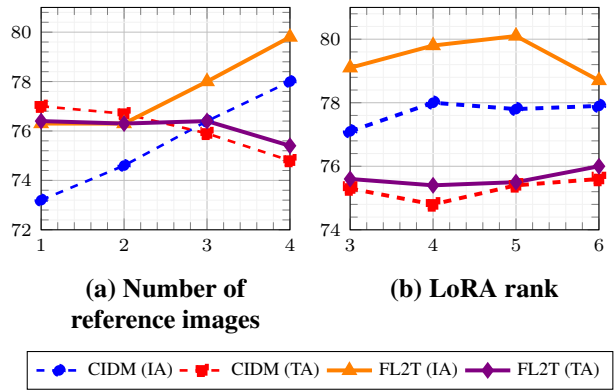

\textbf{Number of transformer decoder layers.} As shown in Table \ref{tab:ablation_tf_layers}, the proxy embeddings achieve optimal inter-concept semantic representation with two transformer layers. This can be attributed to the tendency of excessive attention layers to induce rank collapse, leading to a rank-1 matrix and reduced representational diversity \cite{rank_collapse}. On the other hand, a single transformer layer is too small to capture feature relations. \textit{FL2T achieves maximal performance with two transformer decoder layers.}

\begin{table}[!h]
    \centering
    \resizebox{0.45\textwidth}{!}{
    \begin{tabular}{c|cc|cc|cc}
        \hline
        \hline
        \multirow{2}{*}{\shortstack{\textbf{Number} \\ \textbf{of layers}}} & \multicolumn{2}{|c|}{\textbf{V1 - V5}} & \multicolumn{2}{|c|}{\textbf{V6 - V10}} & \multicolumn{2}{c}{\textbf{Avg.}} \\
        \cline{2-7}
        
        & \textbf{IA} ($\uparrow$) 
         & \cellcolor{lightblue}\textbf{TA} ($\uparrow$)  
         & \textbf{IA} ($\uparrow$) 
         & \cellcolor{lightblue}\textbf{TA} ($\uparrow$) 
         & \textbf{IA} ($\uparrow$) 
         & \cellcolor{lightblue}\textbf{TA} ($\uparrow$)\\
         \hline
         
         1 
         & 83.5
         &\cellcolor{lightblue}77.3
         & 73.4 
         &\cellcolor{lightblue}\underline{74.1}
         & 78.5 
         & \cellcolor{lightblue}\underline{75.7}\\
         
         2 
         & \textbf{84.9} 
         &\cellcolor{lightblue}77.0
         & \textbf{74.8}
         &\cellcolor{lightblue}73.8
         &\textbf{79.8} 
         & \cellcolor{lightblue}75.4\\
         
         3 
         & \textbf{84.9} 
         &\cellcolor{lightblue}76.8
         & 74.0 
         &\cellcolor{lightblue}73.7
         &79.4 
         & \cellcolor{lightblue}75.3\\
         
         4 
         & 83.7 
         &\cellcolor{lightblue}\underline{77.6}
         & 73.4
         &\cellcolor{lightblue}72.7
         &78.6 
         & \cellcolor{lightblue}75.2\\
         \hline
         \shortstack{\textbf{Number} \\ \textbf{of layers}} & \textbf{IMS} ($\uparrow$) 
         & \cellcolor{lightblue}\textbf{FID} ($\downarrow$)  
         & \textbf{IMS} ($\uparrow$) 
         & \cellcolor{lightblue}\textbf{FID} ($\downarrow$) 
         & \textbf{IMS} ($\uparrow$) 
         & \cellcolor{lightblue}\textbf{FID} ($\downarrow$)\\
         \hline
         1 & 75.3
         &\cellcolor{lightblue}139.8
         & 63.4
         &\cellcolor{lightblue}262.7
         & 69.4
         &\cellcolor{lightblue}201.2\\

         2 & \textbf{78.2}
         &\cellcolor{lightblue}133.3
         & \textbf{64.4}
         &\cellcolor{lightblue}262.7
        & \textbf{71.3}
         &\cellcolor{lightblue}198.0\\

         3 & 77.3
         &\cellcolor{lightblue}\underline{132.2}
         & 63.7
         &\cellcolor{lightblue}\underline{261.0}
         & 70.5
         &\cellcolor{lightblue}\underline{196.6}\\

         4 & 76.5
         &\cellcolor{lightblue}139.3
         & 62.9
         &\cellcolor{lightblue}265.2
         & 69.7
         &\cellcolor{lightblue}202.3 \\
         \hline
         \hline
    \end{tabular}}
    \caption{Comparison of model performances based on the number of transformer layers in the concept aggregation module. The best IA and IMS scores have been denoted in \textbf{bold}, and TA and FID values have been \underline{underlined}.}
    \label{tab:ablation_tf_layers}
\end{table}

\textbf{Variation of LoRA parameters.} Fig. \ref{fig:ablation_ref_images_lora}(b) and Table \ref{tab:lora} illustrate the impact of varying LoRA parameters on model performance. Across all ranks of the LoRA matrix, FL2T (Ours) consistently outperforms CIDM \cite{cidm} across all scores. Notably, our framework achieves superior performance compared to CIDM (rank = 4) while utilizing a lower rank or 25\% fewer parameters. \textit{FL2T exhibits greater parameter efficiency by outperforming SOTA models while using a lower LoRA rank.}

\textbf{On the number of concepts.} As shown in Fig. \ref{fig:scalability}, FL2T (Ours) exhibits impressive scalability compared to CIDM \cite{cidm}. This further consolidates our hypothesis that positive inter-concept interactions can help improve generative capabilities. \textit{FL2T demonstrates impressive scalability.}

\pgfplotsset{compat=1.15}
\begin{figure}[!ht]
  \centering
  \begin{tikzpicture}
    \begin{axis}[
      ybar,
      bar width=8pt,
      width=0.5\textwidth,
      height=0.45\textwidth,
      xlabel={Concept count range},
      symbolic x coords={1-10, 11-20, 21-30},
      xtick=data,
      ymin=71, ymax=100,
      grid=both,
      font=\bfseries,
      legend pos=north west,
      legend style={font=\tiny, legend columns=3},
      enlarge x limits=0.25,
      ymajorgrids=true,
      nodes near coords,
      every node near coord/.append style={font=\tiny, yshift=8pt, xshift =4.5pt, text=black, rotate = 90}
    ]
      \addplot[style={red, fill=red!30}] coordinates {
        (1-10, 79.6)
        (11-20, 80.9)
        (21-30, 80.9)
      };
      \addlegendentry{CIDM (IA)}

      \addplot[style={blue, fill=blue!30}] coordinates {
        (1-10, 77.1)
        (11-20, 74.5)
        (21-30, 77.3)
      };
      \addlegendentry{CIDM (TA)}

      \addplot[style={orange, fill=orange!30}] coordinates {
      (1-10, 82.9)
        (11-20, 83.2)
        (21-30, 84.5)
    };
    \addlegendentry{FL2T (IA)}

    \addplot[style={violet, fill=violet!30}] coordinates {
    (1-10, 78.2)
        (11-20, 75.4)
        (21-30, 78.1)
    };
    \addlegendentry{FL2T (TA)}
    
    \addplot[style={green, fill=green!30}] coordinates {
    (1-10, 74.7)
        (11-20, 72.5)
        (21-30, 72.9)
    };
    \addlegendentry{CIDM (IMS)}
    \addplot[style={teal, fill=teal!30}] coordinates {
    (1-10, 76.9)
        (11-20, 75.4)
        (21-30, 77.8)
    };
    \addlegendentry{FL2T (IMS)}
    \end{axis}

    \begin{axis}[
      width=0.5\textwidth,
      height=0.45\textwidth,
      axis y line*=right,
      axis x line=none,
      font = \bfseries,
      ymin=0, ymax=200, 
      symbolic x coords={1-10, 11-20, 21-30},
      xtick=data,
      legend style={font=\tiny, legend columns=1, at={(0.98,0.95)}, anchor=north east},
    ]
      \addplot[mark=square*, violet, line width =1pt] coordinates {
      (1-10, 123.0)
        (11-20, 136.9)
        (21-30, 127.9)
      };
      \addlegendentry{CIDM (FID)}

      \addplot[mark=triangle*, brown, line width =1pt] coordinates {
        (1-10, 109.0)
        (11-20, 122.3)
        (21-30, 109.7)
      };
      \addlegendentry{FL2T (FID)}
    \end{axis}
  \end{tikzpicture}
  \caption{\textbf{Scalability.} FL2T shows impressive scalability achieving higher average CLIP and IMS scores (left axis) and lower FID values (right axis) over three concept ranges (1–10, 11–20, 21–30), compared to CIDM \cite{cidm} on the ImageNet dataset \cite{imagenet}. We limit the scalability analysis to 30 concepts due to CLIP tokenizer's limit on 77 new tokens.}
  \label{fig:scalability}
\end{figure}
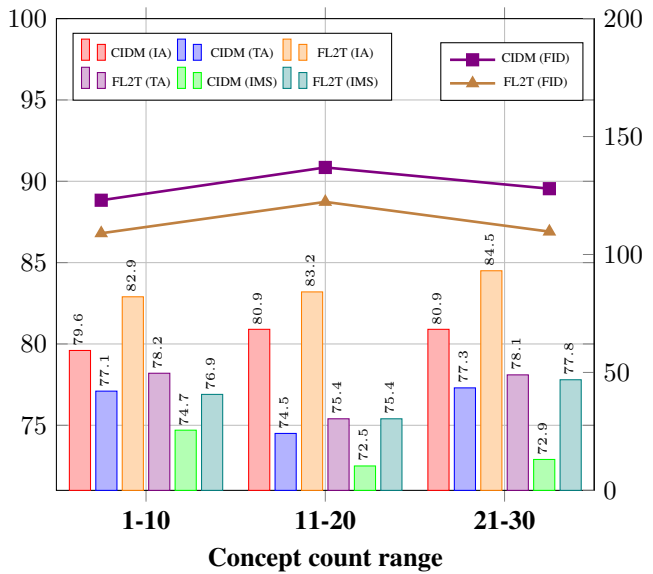


\subsection{Quantitative Analysis}

As shown in Table \ref{tab:main_exp}, we begin our experiments by evaluating our framework under the condition where 9 out of 10 concepts are exposed for a task \( g \), referred to as ``Ours w/o guidance." We then introduce cosine similarity-based guidance to regulate cross-concept interactions using the concept embeddings learned in Step 1 (Section \ref{subsec:step1}), which is denoted as "Ours with Cosine-guidance." This adjustment leads to a notable improvement of +0.8 in the average Image Alignment (IA) scores, with only a negligible decrease in the average Text Alignment (TA) scores. Notably, Concept V6 experiences a substantial 10-point improvement in IA scores. Inspired by this observation, we designed a module that controls cross-concept interactions using proxy embeddings. The results demonstrate that our framework outperforms CIDM \cite{cidm} in at least 6 out of 10 concepts across both IA and TA scores. The smallest gain observed is +0.6, while the largest gain is +9.2. Furthermore, our framework effectively supports the paradigm of personalization and mitigates catastrophic forgetting, as evidenced by a +1.8 improvement in average IA scores and +0.6 in average TA scores. We would also like to highlight the fact that FL2T showcases superior performance in terms of LoRA parameter efficiency (Fig. \ref{fig:ablation_ref_images_lora} (b)) and fewer reference images required for optimal performance (Fig. \ref{fig:ablation_ref_images_lora}(a)).

\begin{figure*}[!h]
    \centering
    \includegraphics[width=\textwidth]{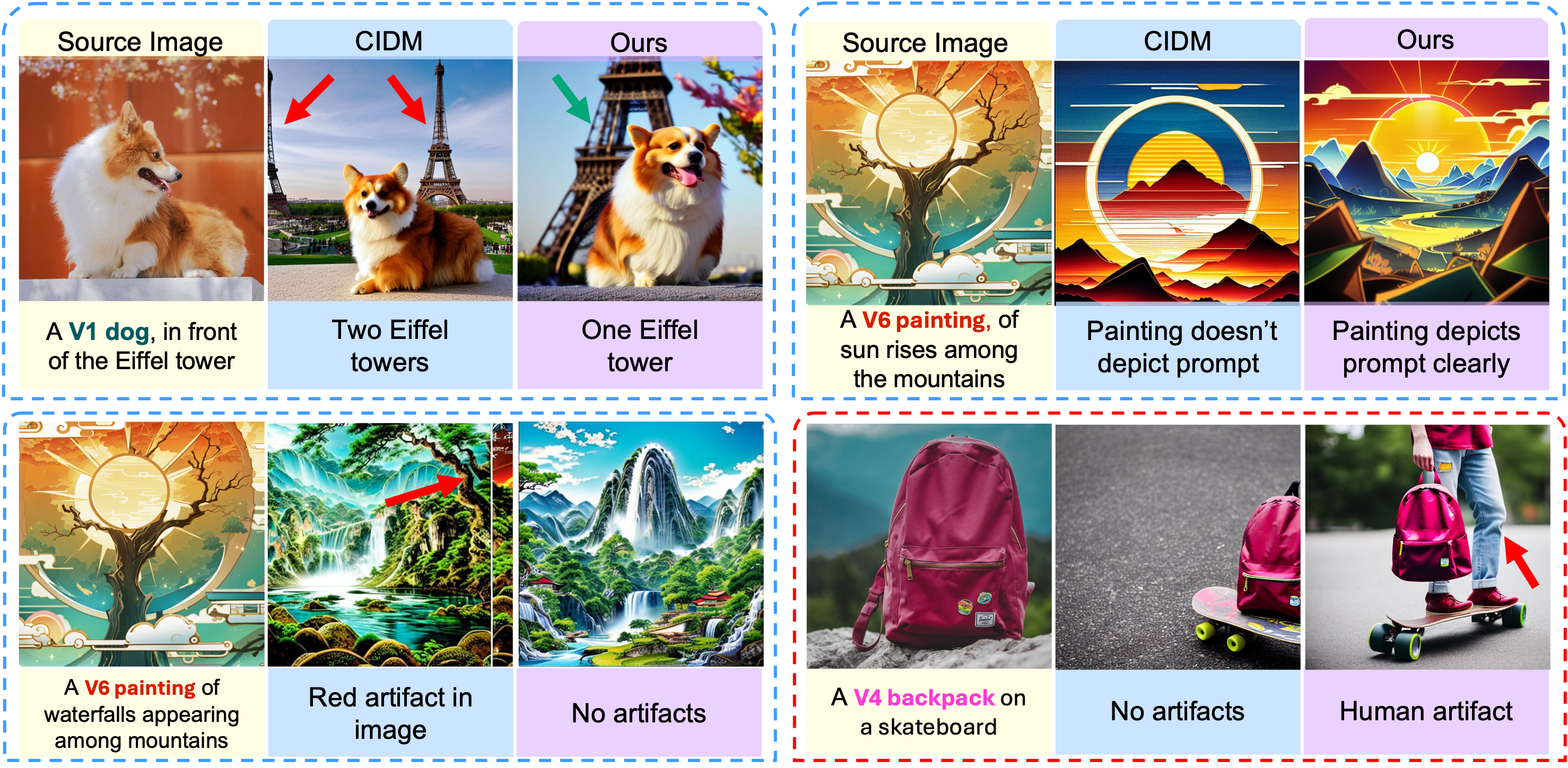}
    \caption{Qualitative Analysis on the CIFC dataset \cite{cidm}. We compare the synthesized images by CIDM \cite{cidm} and FL2T (Ours). The images are generated with a source image and an associated text prompt as input. Images with red and green arrows indicate regions of undesirable and desirable qualities, and their reasons are stated below each image. }
    \label{fig:qual_results1}
\end{figure*}

\subsection{Qualitative Analysis}

As demonstrated in Fig. \ref{fig:qual_results1}, our proposed framework excels in preserving identity, evidenced by the V7 concept generation task. It is also adept at controlling the number of objects in the generated images, as shown in V1 dog and V8 drawing concept images (Appendix Fig. \ref{fig:qual_results2}). Lastly, we also observe a substantial improvement in correcting extraneous features and background generation capabilities. We attribute this improvement to the effective cross-concept interactions, which lead to the generation of more robust concept embeddings. Furthermore, the use of proxy-guidance aids in achieving smoother knowledge transfer and retention across concepts, as illustrated by the enhanced quality of the generated images.

\subsection{Computational Cost Analysis}
We evaluate the computational efficiency of our method by comparing the parameter count, time complexity, memory consumption, and inference latency against CIDM \cite{cidm}. FL2T (Ours) has a higher parameter count (52.6M) compared to CIDM \cite{cidm} (38.4M), due to its proxy-guided inter-concept interaction architecture. Our method exhibits improved performance with a time complexity of $\mathcal{O}(2G)$, compared to CIDM’s $\mathcal{O}(G)$, where $G$ is the number of concepts. Both methods exhibit identical memory consumption of 18.7 GB during inference. Additionally, the inference times for processing 50 images are comparable: 1423 seconds for CIDM and 1436 seconds for FL2T, highlighting that the increased parameter count does not lead to a substantial increase in runtime.

\section{Limitations}
Despite its strengths, FL2T exhibits a few limitations. First, it occasionally fails to accurately generate the prompted context, which can be attributed to weak priors for certain concepts or low co-occurrence between the subject and context in the training data (e.g., as seen in row 2, column 6 of Fig. \ref{fig:qual_results1}. Second, the model may overfit to real images when the prompt closely mirrors conditions seen during training, limiting its generalization. Lastly, FL2T’s ability to learn subject identities varies significantly - while common subjects such as dogs and cats are modeled effectively, rare or less-represented subjects remain challenging.

\section{Conclusion}
Our work tackles the limitations of existing CDMs, particularly their susceptibility to catastrophic forgetting when learning new concepts. Unlike prior approaches that assume a fixed order of concept learning and neglect inter-concept interactions, we introduce Forget Less by Learning Together (FL2T) — a novel framework that enables order-agnostic concept learning while mitigating forgetting. At the core of FL2T is an inter-concept learning mechanism which leverages the permutation invariance property of transformers. In this mechanism, proxy embeddings guide feature selection across concepts, ensuring both knowledge retention and efficient adaptation. By leveraging inter-concept guidance, our method preserves previously learned concepts while seamlessly integrating new ones.
{\small
\bibliographystyle{ieee_fullname}
\bibliography{egbib}
}

\newpage
\appendix

\definecolor{lightred}{RGB}{255, 204, 204}
\definecolor{lightblue}{RGB}{204, 229, 255}
\definecolor{lightgreen}{RGB}{204, 255, 204}
\definecolor{lightyellow}{RGB}{255, 255, 204}
\definecolor{lightgray}{gray}{0.9}
\definecolor{lightorange}{RGB}{255, 230, 153}


\crefname{section}{Sec.}{Secs.}
\Crefname{section}{Section}{Sections}
\Crefname{table}{Table}{Tables}
\crefname{table}{Tab.}{Tabs.}



\def\wacvPaperID{1594} 
\def\confName{WACV}
\def\confYear{2024}

\appendix
\section{Technical Appendices and Supplementary Material}

\subsection{Ablation study}
\noindent We conduct an ablation study comparing FL2T (Ours) and CIDM \cite{cidm} under varying settings. First, constraining the number of reference images (Table \ref{tab:ref_images}) reveals that while both models improve with more references, FL2T consistently outperforms CIDM and matches its original performance ($\geq 4$ images) using only three, demonstrating greater efficiency. Second, as shown in Table \ref{tab:ablation_tf_layers} proxy embeddings attain optimal semantic representation with two transformer decoder layers, as too few layers under-express feature relations and too many risk rank collapse \cite{rank_collapse}. Lastly, FL2T maintains superior scores across all LoRA ranks (Table \ref{tab:lora}), outperforming CIDM (rank = 4) even at a lower rank  with 25\% fewer parameters, highlighting the parameter efficiency of our approach.

\begin{table}[!h]
    \centering
    \renewcommand{\arraystretch}{1.2}
    \resizebox{0.45\textwidth}{!}{
    \begin{tabular}{c|c|cc|cc|cc}
        \hline
        \hline
         \multirow{2}{*}{\shortstack{\textbf{Number of} \\ \textbf{reference} \\ \textbf{images}}} & \multirow{2}{*}{\textbf{Methods}} & \multicolumn{2}{c}{\textbf{V1 - V5}} & \multicolumn{2}{|c|}{\textbf{V6 - V10}} & \multicolumn{2}{c}{\textbf{Avg.}} \\
         \cline{3-8}
         
         & & \textbf{IA} ($\uparrow$) 
         & \cellcolor{lightblue}\textbf{TA} ($\uparrow$) 
         & \textbf{IA} ($\uparrow$) 
         & \cellcolor{lightblue}\textbf{TA} ($\uparrow$)  
         & \textbf{IA} ($\uparrow$) 
         & \cellcolor{lightblue}\textbf{TA} ($\uparrow$)\\
         \hline
         
         \multirow{2}{*}{1} & CIDM 
         & 76.6 
         &\cellcolor{lightblue}\underline{79.5} 
         & 69.7
        &\cellcolor{lightblue}\underline{74.5} 
        & 73.2  
        & \cellcolor{lightblue} \underline{77.0} \\
         
          & Ours 
          & \textbf{80.0} 
          &\cellcolor{lightblue}78.7 
          & \textbf{72.5} 
          & \cellcolor{lightblue}74.1 
          &\textbf{76.3}  
          & \cellcolor{lightblue}76.4 \\

          \hline
          
         \multirow{2}{*}{2} & CIDM 
         & 78.1 
         &\cellcolor{lightblue}\underline{79.2} 
         & 71.2
          &\cellcolor{lightblue}\underline{74.2} 
          & 74.6 
          & \cellcolor{lightblue}\underline{76.7}\\

           & Ours
           & \textbf{80.4} 
           &\cellcolor{lightblue}78.7 
           & \textbf{72.1}
           &\cellcolor{lightblue}74.0 
           & \textbf{76.3} 
           & \cellcolor{lightblue} 76.3 \\

          \hline

          \multirow{2}{*}{3} & CIDM 
          & 80.8 
          &\cellcolor{lightblue}78.2 
          & 72.1 
          &\cellcolor{lightblue}73.6 
          & 76.4 
          & \cellcolor{lightblue}75.9 \\
          
           & Ours 
           & \textbf{82.3} 
           &\cellcolor{lightblue} \underline{79.0} 
           & \textbf{73.7}  
           & \cellcolor{lightblue}\underline{73.9}
           & \textbf{78.0}  
           & \cellcolor{lightblue}\underline{76.4}\\

           \hline

          \multirow{2}{*}{$\geq$4} & CIDM 
          & 84.0 
          &\cellcolor{lightblue}\underline{77.3} 
          & 71.7
          &\cellcolor{lightblue}72.6 
          & 78.0 
          & \cellcolor{lightblue}74.8 \\
         
         & Ours 
         & \textbf{84.9} 
         & \cellcolor{lightblue}77.0 
         & \textbf{74.8} 
         & \cellcolor{lightblue}\underline{88.9} 
         & \textbf{79.8} 
         & \cellcolor{lightblue}\underline{75.4}\\
         \hline
         \hline
         
         \shortstack{\textbf{Number of} \\ \textbf{reference}} & \textbf{Methods} & \textbf{IMS} ($\uparrow$) 
         & \cellcolor{lightblue}\textbf{FID} ($\downarrow$) 
         & \textbf{IMS} ($\uparrow$) 
         & \cellcolor{lightblue}\textbf{FID} ($\downarrow$)  
         & \textbf{IMS} ($\uparrow$) 
         & \cellcolor{lightblue}\textbf{FID} ($\downarrow$)\\
         \hline
         \multirow{2}{*}{1} & CIDM
         & 61.7
         & \cellcolor{lightblue}199.0
         & \textbf{55.3}
        & \cellcolor{lightblue} \underline{291.8}
        & 58.5
        & \cellcolor{lightblue}245.4\\

        & Ours
        & \textbf{65.0}
          & \cellcolor{lightblue}\underline{179.2}
          & 54.7
          & \cellcolor{lightblue}292.4
          & \textbf{59.9}
          &\cellcolor{lightblue}\underline{235.8}\\
        \hline

        \multirow{2}{*}{2} & CIDM 
        & 66.5
         & \cellcolor{lightblue} 178.3
         & 57.8
          & \cellcolor{lightblue} 284.9
          & 62.2
          & \cellcolor{lightblue}231.6\\

          & Ours
          & \textbf{69.0}
           & \cellcolor{lightblue}\underline{165.5}
           & \textbf{58.1}
           & \cellcolor{lightblue}\underline{278.8}
           & \textbf{63.6}
           &\cellcolor{lightblue}\underline{222.1}\\

           \hline

          \multirow{2}{*}{3} & CIDM
          & 71.3
          & \cellcolor{lightblue}156.5
          & 58.5
          & \cellcolor{lightblue}285.3
          & 64.9
          & \cellcolor{lightblue}220.9\\

          & Ours
          & \textbf{72.0}
           & \cellcolor{lightblue}\underline{152.9}
           & \textbf{60.3}
           & \cellcolor{lightblue}\underline{272.7}
           & \textbf{66.2}
           &\cellcolor{lightblue} \underline{212.8}\\

           \hline

           \multirow{2}{*}{$\geq$4} & CIDM 
           & 74.0
          & \cellcolor{lightblue}142.0
          & 61.5
          & \cellcolor{lightblue}270.0
          & 67.7
          & \cellcolor{lightblue}206.0\\

          & Ours
          & \textbf{78.2}
         & \cellcolor{lightblue}\underline{133.3}
         & \textbf{64.4}
         & \cellcolor{lightblue}\underline{262.7}
         & \textbf{71.3}
         & \cellcolor{lightblue}\underline{198.0}\\
         \hline
         \hline
    \end{tabular}}
    \caption{Comparison of model performances with different number of reference images. IA and TA refer to CLIP \cite{clip} Image Alignment and Text Alignment scores respectively. The best IA and IMS scores have been denoted in \textbf{bold}, and the TA and FID values have been \underline{underlined}.}
    \label{tab:ref_images}
\end{table}

\begin{table}[!h]
    \centering
    \renewcommand{\arraystretch}{1.2}
    \resizebox{0.45\textwidth}{!}{
    \begin{tabular}{c|c|cc|cc|cc}
        \hline
        \hline
         \multirow{2}{*}{\shortstack{\textbf{LoRA} \\ \textbf{rank}}} & \multirow{2}{*}{\textbf{Methods}} & \multicolumn{2}{c}{\textbf{V1 - V5}} & \multicolumn{2}{|c|}{\textbf{V6 - V10}} & \multicolumn{2}{c}{\textbf{Avg.}} \\
         \cline{3-8}
         
         & & \textbf{IA} ($\uparrow$) 
         & \cellcolor{lightblue}\textbf{TA} ($\uparrow$) 
         & \textbf{IA} ($\uparrow$) 
         & \cellcolor{lightblue}\textbf{TA} ($\uparrow$) 
         & \textbf{IA} ($\uparrow$) 
         & \cellcolor{lightblue}\textbf{TA} ($\uparrow$) \\
         \hline
         
         \multirow{2}{*}{3} & CIDM 
         & 82.3 
         &\cellcolor{lightblue}\underline{78.3}
         & 71.9 
         &\cellcolor{lightblue}72.4
         & 77.1 
         & \cellcolor{lightblue}75.3\\
         
         & Ours 
         & \textbf{84.5} 
         & \cellcolor{lightblue}77.6 
         & \textbf{73.8} 
         & \cellcolor{lightblue}\underline{88.6} 
         & \textbf{79.1} 
         & \cellcolor{lightblue}\underline{75.6}\\
         \hline
         
         \multirow{2}{*}{4} & CIDM 
         & 84.0 
         &\cellcolor{lightblue}\underline{77.3}
         & 71.7
         &\cellcolor{lightblue}72.6 
         & 78.0 
         & \cellcolor{lightblue}74.8\\
         
         & Ours 
         & \textbf{84.9} 
         & \cellcolor{lightblue}77.0
         & \textbf{74.8} 
         & \cellcolor{lightblue}\underline{88.9}
         & \textbf{79.8} 
         & \cellcolor{lightblue}\underline{75.4}\\
         \hline
         
         \multirow{2}{*}{5} & CIDM 
         & 82.5 
         &\cellcolor{lightblue}\underline{78.1}
         & 73.1
         &\cellcolor{lightblue} 72.9
         & 77.8 
         & \cellcolor{lightblue}75.4\\
         
         & Ours 
         & \textbf{85.4} 
         & \cellcolor{lightblue}77.3
         & \textbf{74.7} 
         & \cellcolor{lightblue}\underline{73.7}
         & \textbf{80.1} 
         & \cellcolor{lightblue}\underline{75.5}\\
         \hline
         
         \multirow{2}{*}{6} & CIDM 
         & 82.7 
         &\cellcolor{lightblue}\underline{78.3} 
         & 73.0 
         &\cellcolor{lightblue}73.0
         & 77.9 
         & \cellcolor{lightblue}75.6\\
         
         & Ours 
         & \textbf{84.4} 
         & \cellcolor{lightblue}77.7
         & \textbf{73.1} 
         & \cellcolor{lightblue}\underline{74.2}
         & \textbf{78.7} 
         & \cellcolor{lightblue}\underline{76.0}\\
         \hline
         \hline
         \shortstack{\textbf{LoRA} \\ \textbf{rank}}& \textbf{Methods} & \textbf{IMS} ($\uparrow$) 
         & \cellcolor{lightblue}\textbf{FID} ($\downarrow$) 
         & \textbf{IMS} ($\uparrow$) 
         & \cellcolor{lightblue}\textbf{FID} ($\downarrow$) 
         & \textbf{IMS} ($\uparrow$) 
         & \cellcolor{lightblue}\textbf{FID} ($\downarrow$) \\
         \hline

         \multirow{2}{*}{3} & CIDM
         & 72.6
         & \cellcolor{lightblue}150.3
         & \textbf{62.2}
         & \cellcolor{lightblue}\underline{262.7}
         & 67.4
         & \cellcolor{lightblue}206.5\\

         & Ours 
         & \textbf{76.6}
         & \cellcolor{lightblue}\underline{135.1}
         & 61.8
         & \cellcolor{lightblue}263.9
         & \textbf{69.2}
         & \cellcolor{lightblue}\underline{199.5}\\
         \hline

          \multirow{2}{*}{4} & CIDM
          & 74.0
         & \cellcolor{lightblue}142.0
         & 61.5
         & \cellcolor{lightblue}270.0
         & 67.7
         & \cellcolor{lightblue}206.0\\
         
         & Ours
         & \textbf{78.2}
         & \cellcolor{lightblue}\underline{133.3}
         & \textbf{64.4}
         & \cellcolor{lightblue}\underline{262.7}& \textbf{71.3}
         & \cellcolor{lightblue}\underline{198.0}\\
         \hline
         
         \multirow{2}{*}{5} & CIDM 
         & 73.5
         & \cellcolor{lightblue}145.4
         & 62.4
         & \cellcolor{lightblue}\underline{260.2}
         & 68.0
         & \cellcolor{lightblue}202.8\\
         
         & Ours
         & \textbf{76.8}
         & \cellcolor{lightblue}\underline{132.5}
         & \textbf{63.3}
         & \cellcolor{lightblue}263.7
         & \textbf{70.1}
         & \cellcolor{lightblue}\underline{198.1}\\
         \hline
         
         \multirow{2}{*}{6} & CIDM 
         & 72.6
         & \cellcolor{lightblue}147.1
         & \textbf{61.7}
         & \cellcolor{lightblue}\underline{267.2}
         & 67.1
         & \cellcolor{lightblue}207.2\\
         
         & Ours 
         & \textbf{76.7}
         & \cellcolor{lightblue}\underline{134.8}
         & 60.9
         & \cellcolor{lightblue}273.3
         & \textbf{68.9}
         & \cellcolor{lightblue}\underline{204.0}\\
         \hline
         \hline
    \end{tabular}}
    \caption{Comparison of model performances for different LoRA \cite{lora1, lora2} configurations. IA and TA refer to CLIP \cite{clip} Image Alignment and Text Alignment scores, respectively. The better IA and IMS scores have been denoted in \textbf{bold}, and TA and FID values have been \underline{underlined}.}
    \label{tab:lora}
\end{table}

\subsection{Datasets, Implementation Details and Evaluation Metrics}
\paragraph{Datasets.} We evaluate our approach on three datasets, each selected to ensure strong semantic alignment with our task objectives (Fig. \ref{fig:datasets}). For the CIFC dataset \cite{cidm}, we adopt the benchmark established in the original paper, which features visually rich concept cards designed to challenge models in hallucination detection. For the ImageNet \cite{imagenet} subset, we manually select 3–5 images per class from ten classes, ensuring that each image contains a well-centered, unobstructed primary object. For the CelebA dataset \cite{celeba}, we construct a subset of ten identities, with 3–5 representative images per identity. Images are chosen based on clear frontal facial orientation, uniform lighting, and minimal occlusion to preserve identity consistency. This hand-curated dataset design provides high-quality supervision, which is crucial for minimizing semantic hallucinations and promoting accurate visual-textual alignment. 

For the experiments on varying numbers of reference images (Table \ref{tab:ref_images}), we have used a fixed set of images picked randomly (without human intervention).

\textbf{Implementation Details.} We use Stable Diffusion (SD-1.5) \cite{finetune2} as the pretrained model for all experiments. The training is conducted with a fixed initial learning rate of \(1.0 \times 10^{-3}\) for updating textual embeddings and \(1.0 \times 10^{-4}\) for optimizing the U-Net. Empirically, we set \(\gamma_1 = 0.1\) and \(\gamma_2 = 0.1\) in Eq. \ref{eq:cil}.

\begin{table*}[!t]
    \centering
    \renewcommand{\arraystretch}{1.5}
    \resizebox{\textwidth}{!}{
    \begin{tabular}{cc|cc|cc|cc|cc|cc|cc|cc|cc|cc|cc|cc}
        \hline
        \hline
        & \multirow{2}{*}{\textbf{Methods}} & \multicolumn{2}{|c|}{\textbf{V1}} & \multicolumn{2}{|c|}{\textbf{V2}} & \multicolumn{2}{|c|}{\textbf{V3}} & \multicolumn{2}{|c|}{\textbf{V4}} & \multicolumn{2}{|c|}{\textbf{V5}} & \multicolumn{2}{|c|}{\textbf{V6}} & \multicolumn{2}{|c|}{\textbf{V7}} & \multicolumn{2}{|c|}{\textbf{V8}} & \multicolumn{2}{|c|}{\textbf{V9}} & \multicolumn{2}{|c|}{\textbf{V10}} & \multicolumn{2}{|c}{\textbf{Avg.}} \\
        \cline{3-24}  
         & & \textbf{IMS} ($\uparrow$) & \cellcolor{lightblue}\textbf{FID} ($\downarrow$)  & \textbf{IMS} ($\uparrow$) & \cellcolor{lightblue}\textbf{FID} ($\downarrow$) & \textbf{IMS} ($\uparrow$) & \cellcolor{lightblue}\textbf{FID} ($\downarrow$) & \textbf{IMS} ($\uparrow$) & \cellcolor{lightblue}\textbf{FID} ($\downarrow$) & \textbf{IMS} ($\uparrow$) & \cellcolor{lightblue}\textbf{FID} ($\downarrow$) & \textbf{IMS} ($\uparrow$) & \cellcolor{lightblue}\textbf{FID} ($\downarrow$) & \textbf{IMS} ($\uparrow$) & \cellcolor{lightblue}\textbf{FID} ($\downarrow$) & \textbf{IMS} ($\uparrow$) & \cellcolor{lightblue}\textbf{FID} ($\downarrow$) & \textbf{IMS} ($\uparrow$) & \cellcolor{lightblue}\textbf{FID} ($\downarrow$) & \textbf{IMS} ($\uparrow$) & \cellcolor{lightblue}\textbf{FID} ($\downarrow$) & \textbf{IMS} ($\uparrow$) & \cellcolor{lightblue}\textbf{FID} ($\downarrow$) \\
         \hline

         \cellcolor{lightyellow} & \textbf{CIDM}\cite{cidm} & 
          84.7 & \cellcolor{lightblue}82.2 &
          74.5 & \cellcolor{lightblue}166.1 &
          66.7 & \cellcolor{lightblue}156.7 &
          65.1 & \cellcolor{lightblue}211.9 &
          78.9 & \cellcolor{lightblue}93.2 &
          53.6 & \cellcolor{lightblue}348.4 &
          80.8 & \cellcolor{lightblue}119.0 &
          49.7 & \cellcolor{lightblue}345.6 &
          70.0 & \cellcolor{lightblue}195.6 &
          53.3 & \cellcolor{lightblue}341.5 &
          67.7 & \cellcolor{lightblue} 206.0 \\
         
         \cellcolor{lightyellow} & \shortstack{\textbf{Ours} with \\ proxy-guidance} &
         87.7& \cellcolor{lightblue}71.8 &
         75.1& \cellcolor{lightblue} 169.3 &
         75.6& \cellcolor{lightblue} 145.4 &
         71.3& \cellcolor{lightblue} 202.5&
         81.1& \cellcolor{lightblue} 77.7 &
         55.8& \cellcolor{lightblue} 355.2&
         89.5& \cellcolor{lightblue} 98.4&
         48.3& \cellcolor{lightblue} 353.4&
         69.8& \cellcolor{lightblue} 189.9&
         58.6& \cellcolor{lightblue} 316.4&
         71.3 & \cellcolor{lightblue} 198.0\\
        
        \rowcolor{lightgray} \multirow{-3}{*}{\cellcolor{lightyellow}\rotatebox{90}{\textbf{CIFC \cite{cidm}}}}  &
         \textbf{$\Delta$} & 
           +3.0& -10.4 &  
           +0.6& +3.2 &
           +8.9& -11.3 & 
           +6.2& -8.6 & 
           +2.2& -15.5 &
           +2.2& +6.8 & 
           +8.7& -20.6 & 
           -1.4& +7.8 & 
           -0.2& -5.7 &
           +5.3& -25.1 &
           +3.6 & -8.0 \\
         \hline
         \hline
         
          
         \cellcolor{lightgreen} & \textbf{CIDM}\cite{cidm} & 
          63.9& \cellcolor{lightblue} 233.2 &
          65.6& \cellcolor{lightblue} 219.7&
          60.1& \cellcolor{lightblue} 250.0&
          76.0& \cellcolor{lightblue} 166.8&
          55.3& \cellcolor{lightblue} 246.7&
          61.4& \cellcolor{lightblue} 234.5&
          58.0& \cellcolor{lightblue} 193.4&
          56.1& \cellcolor{lightblue} 273.5&
          51.2& \cellcolor{lightblue} 185.7&
          57.0& \cellcolor{lightblue} 197.0&
          60.5  & \cellcolor{lightblue} 220.0\\
         
         \cellcolor{lightgreen} & \shortstack{\textbf{Ours} with \\ proxy-guidance} &
         64.4& \cellcolor{lightblue}225.8 &
         69.3& \cellcolor{lightblue}212.2 &
         63.5& \cellcolor{lightblue}239.8 &
         70.2& \cellcolor{lightblue}168.9 &
         57.4& \cellcolor{lightblue}188.7 &
         63.9& \cellcolor{lightblue}221.1 &
         58.1&  \cellcolor{lightblue} 218.0&
         59.5&  \cellcolor{lightblue}236.9&
         55.4& \cellcolor{lightblue} 170.4 &
         63.0&  \cellcolor{lightblue} 195.4&
         62.5 & \cellcolor{lightblue}209.0 \\
        
        \rowcolor{lightgray} \multirow{-3}{*}{\cellcolor{lightgreen}\rotatebox{90}{\textbf{CelebA \cite{celeba}}}}  &
         \textbf{$\Delta$} & 
           +0.5& -7.4 &
           +3.7& -7.5&  
           +3.4& -10.2&  
           -5.8& +2.1&  
           +2.1& -58.0& 
           +2.5& -13.4&  
           +0.1& +24.6&  
           +3.4& -36.6&  
           +4.2& -15.3& 
           +6.0& -1.6&  
           +2.0  & -11.0\\
         \hline
         \hline

          
         \cellcolor{lightorange} & \textbf{CIDM}\cite{cidm} & 
          81.6& \cellcolor{lightblue} 129.6 &
          68.3& \cellcolor{lightblue} 113.6 &
          85.4& \cellcolor{lightblue} 75.7 &
          85.0& \cellcolor{lightblue} 85.1 &
          73.5& \cellcolor{lightblue} 91.5 &
          78.4& \cellcolor{lightblue} 86.2 &
          64.9& \cellcolor{lightblue} 135.1&
          79.4& \cellcolor{lightblue} 58.8 &
          80.1& \cellcolor{lightblue} 88.1& 
          87.4& \cellcolor{lightblue} 57.3 &
          78.4 & \cellcolor{lightblue} 92.1 \\
         
         \cellcolor{lightorange} & \shortstack{\textbf{Ours} with \\ proxy-guidance} &
         85.2& \cellcolor{lightblue} 103.1&
         75.2& \cellcolor{lightblue} 107.1 & 
         88.5& \cellcolor{lightblue} 57.1&
         84.1& \cellcolor{lightblue} 79.3&
         85.7& \cellcolor{lightblue} 83.5& 
         84.8& \cellcolor{lightblue} 69.1& 
         75.6& \cellcolor{lightblue} 110.3&
         77.2& \cellcolor{lightblue} 64.3& 
         80.2& \cellcolor{lightblue}79.8&
         95.4& \cellcolor{lightblue}31.6&
         83.2 & \cellcolor{lightblue}78.5\\
        
        \rowcolor{lightgray} \multirow{-3}{*}{\cellcolor{lightorange}\rotatebox{90}{\textbf{INet \cite{imagenet}}}}  &
         \textbf{$\Delta$} & 
           +3.6& -26.5& 
           +7.1& -6.5 &  
           +3.1& -18.6& 
           -0.9& -6.2&  
           +12.2& -8.0& 
           +6.4& -17.1&  
           +10.7& -24.8&  
           -2.2& +5.5&  
           +0.1& -8.3&  
           +8.0& -25.7&  
           +4.8 & -13.6\\
         \hline
         \hline
    \end{tabular}
    }
    \caption{\textbf{Additional Experiments.} FL2T outperforms CIDM across the three datasets on Identity Matching Scores (IMS) \cite{liu2024metacloakpreventingunauthorizedsubjectdriven} and Fretchet Inception Distance (FID) \cite{fid}, effectively showcasing its ability to preserve identities better and improved generation capabilities. }
    \label{tab:addl_exp}
\end{table*}

\textbf{Evaluation Metrics.}
Following the experiments under our problem setting as in Sec \ref{sec:problem}, we evaluate our generated images across 2 metrics - Image Alignment (IA) and Text Alignment (TA). Image Alignment (IA) scores are computed using CLIP \cite{clip} image encoder, comparing the similarity of features between generated images and reference images. Similarly, we utilize the text encoder of CLIP \cite{clip} to evaluate the text-image similarity between the input prompt and synthesized image for the Text Alignment (TA) scores. Additionally, we utilize \textit{Identity Matching Score (IMS)} that measures the semantic closeness of the generated image and reference image \cite{liu2024metacloakpreventingunauthorizedsubjectdriven}. It is computed as the cosine similarity score between embeddings of generated images and mean of reference image embeddings. For the CIFC dataset \cite{cidm} and the ImageNet dataset \cite{imagenet}, we utilize ResNet-152 \cite{resnet} as the image encoder. On the other hand, we have utilized VGG-Face \cite{vggface} for CelebA \cite{celeba}. Further, \textit{Fretchet Inception Distance (FID)} evaluates the quality and diversity of images by comparing the Inception-v3 feature distributions of the reference images and generated images \cite{fid}.

\subsection{Additional Experiments}
\noindent Beyond CLIP scores, we utilize two other metrics,  \textit{Identity Matching Score (IMS)} and \textit{Fretchet Inception Distance (FID)}. IMS that measures the semantic closeness of the generated image and reference image \cite{liu2024metacloakpreventingunauthorizedsubjectdriven}. Specifically, this metric is designed to measure the identity match between of the concepts in generated images. Whereas FID compares the distributions of generated and reference images using Inception-v3 features \cite{fid}. Based on these metrics, as shown in Table \ref{tab:addl_exp}, we observe that FL2T outperforms CIDM across the three datasets, showing trends akin to CLIP \cite{clip} scores. 

On comparison with Multi-concept Customization \cite{Kumari2023MultiConcept}, FL2T delivers consistent identity gains , boosting IMS by 12–31 points across concepts, with strong IA improvements on most concepts. Importantly, it achieves lower (better) FID indicating high quality image generation results as shown in Table \ref{tab:multi_concept}.

\begin{table*}[!h]
    \centering
    \renewcommand{\arraystretch}{1.5}
    \resizebox{\textwidth}{!}{
    \begin{tabular}{cc|cc|cc|cc|cc|cc|cc|cc|cc|cc|cc|cccc}
        \hline\hline
        & \multirow{2}{*}{\textbf{Methods}} 
        & \multicolumn{2}{|c|}{\textbf{V1}} 
        & \multicolumn{2}{|c|}{\textbf{V2}} 
        & \multicolumn{2}{|c|}{\textbf{V3}} 
        & \multicolumn{2}{|c|}{\textbf{V4}} 
        & \multicolumn{2}{|c|}{\textbf{V5}} 
        & \multicolumn{2}{|c|}{\textbf{V6}} 
        & \multicolumn{2}{|c|}{\textbf{V7}} 
        & \multicolumn{2}{|c|}{\textbf{V8}} 
        & \multicolumn{2}{|c|}{\textbf{V9}} 
        & \multicolumn{2}{|c|}{\textbf{V10}} 
        & \multicolumn{2}{|c}{\textbf{Avg.}} \\
        \cline{3-24}  
        & & \textbf{IA} ($\uparrow$) & \cellcolor{lightblue}\textbf{TA} ($\uparrow$) 
          & \textbf{IA} ($\uparrow$) & \cellcolor{lightblue}\textbf{TA} ($\uparrow$) 
          & \textbf{IA} ($\uparrow$) & \cellcolor{lightblue}\textbf{TA} ($\uparrow$) 
          & \textbf{IA} ($\uparrow$) & \cellcolor{lightblue}\textbf{TA} ($\uparrow$) 
          & \textbf{IA} ($\uparrow$) & \cellcolor{lightblue}\textbf{TA} ($\uparrow$) 
          & \textbf{IA} ($\uparrow$) & \cellcolor{lightblue}\textbf{TA} ($\uparrow$)
          & \textbf{IA} ($\uparrow$) & \cellcolor{lightblue}\textbf{TA} ($\uparrow$) 
          & \textbf{IA} ($\uparrow$) & \cellcolor{lightblue}\textbf{TA} ($\uparrow$) 
          & \textbf{IA} ($\uparrow$) & \cellcolor{lightblue}\textbf{TA} ($\uparrow$) 
          & \textbf{IA} ($\uparrow$) & \cellcolor{lightblue}\textbf{TA} ($\uparrow$) & \textbf{IA} ($\uparrow$) & \cellcolor{lightblue}\textbf{TA} ($\uparrow$) \\
        \hline

        & \shortstack{\textbf{Multi-concept}\\ \textbf{Customization}} \cite{Kumari2023MultiConcept} &
        74.4 & \cellcolor{lightblue}\underline{77.6} &
        76.6 & \cellcolor{lightblue}73.4 &
        73.4 & \cellcolor{lightblue}\underline{74.7} &
        69.7 & \cellcolor{lightblue}73.9 &
        78.2 & \cellcolor{lightblue}\underline{80.7} &
        65.9 & \cellcolor{lightblue}\underline{76.0} &
        73.6 & \cellcolor{lightblue}66.3 &
        68.7 & \cellcolor{lightblue}76.2 &
        62.7 & \cellcolor{lightblue}75.8 &
        71.2 & \cellcolor{lightblue}\underline{83.5} &
        71.4 & \cellcolor{lightblue}\underline{75.8} \\

        & \shortstack{\textbf{Ours} with \\ proxy guidance} &
        \textbf{84.4} & \cellcolor{lightblue}74.0 & 
        \textbf{86.1} & \cellcolor{lightblue}\underline{81.6} & 
        \textbf{84.4} & \cellcolor{lightblue}70.1 & 
        \textbf{82.2} & \cellcolor{lightblue}\underline{81.1} & 
        \textbf{87.2} & \cellcolor{lightblue}78.0 & 
        \textbf{69.3} & \cellcolor{lightblue}72.9 & 
        \textbf{85.3} & \cellcolor{lightblue}\underline{70.4} & 
        \textbf{60.7} & \cellcolor{lightblue}\underline{77.6} & 
        \textbf{82.1} & \cellcolor{lightblue}\underline{76.6} & 
        \textbf{76.5} & \cellcolor{lightblue}71.6 & 
        \textbf{79.8} & \cellcolor{lightblue}{75.4} \\
        \rowcolor{lightgray}& $\Delta$ & 
           +10.0 &  -3.6 &
           +9.5  &  +8.2 &  
           +11.0 &  -4.6 & 
           +12.5 &  +7.2 &  
           +9.0  & -2.7 & 
           +3.4  &  -3.1 &  
           +11.7 &  +4.1 &  
           -8.0  &  +1.4 &  
           +19.4 &  +0.8 &  
           +5.3  &  -11.9 &  
           +8.4  &  -0.4 \\
        \hline\hline
        & \textbf{Methods} & \textbf{IMS} ($\uparrow$) & \cellcolor{lightblue}\textbf{FID} ($\downarrow$)  & \textbf{IMS} ($\uparrow$) & \cellcolor{lightblue}\textbf{FID} ($\downarrow$) & \textbf{IMS} ($\uparrow$) & \cellcolor{lightblue}\textbf{FID} ($\downarrow$) & \textbf{IMS} ($\uparrow$) & \cellcolor{lightblue}\textbf{FID} ($\downarrow$) & \textbf{IMS} ($\uparrow$) & \cellcolor{lightblue}\textbf{FID} ($\downarrow$) & \textbf{IMS} ($\uparrow$) & \cellcolor{lightblue}\textbf{FID} ($\downarrow$) & \textbf{IMS} ($\uparrow$) & \cellcolor{lightblue}\textbf{FID} ($\downarrow$) & \textbf{IMS} ($\uparrow$) & \cellcolor{lightblue}\textbf{FID} ($\downarrow$) & \textbf{IMS} ($\uparrow$) & \cellcolor{lightblue}\textbf{FID} ($\downarrow$) & \textbf{IMS} ($\uparrow$) & \cellcolor{lightblue}\textbf{FID} ($\downarrow$) & \textbf{IMS} ($\uparrow$) & \cellcolor{lightblue}\textbf{FID} ($\downarrow$) \\
        \hline
       & \shortstack{\textbf{Multi-concept}\\ \textbf{Customization}} \cite{Kumari2023MultiConcept} & 
           65.4 & \cellcolor{lightblue} \underline{160.9} & 
           62.8 & \cellcolor{lightblue} \underline{98.7} & 
           64.1 & \cellcolor{lightblue} 189.5 & 
           63.7 & \cellcolor{lightblue} 222.6 & 
           66.2 & \cellcolor{lightblue} 86.5 & 
           67.8 & \cellcolor{lightblue} 386.6 & 
           61.9 & \cellcolor{lightblue} \underline{91.2} & 
           60.5 & \cellcolor{lightblue} 386.9 & 
           62.3 & \cellcolor{lightblue} 220.0 & 
           64.8 & \cellcolor{lightblue} 379.3 & 
           63.5 & \cellcolor{lightblue} 216.2 \\

        & \shortstack{\textbf{Ours} with \\ proxy guidance} & 
           \textbf{85.2}& \cellcolor{lightblue} 71.8&
         \textbf{75.2}& \cellcolor{lightblue} 169.3 & 
         \textbf{88.5}& \cellcolor{lightblue} \underline{145.4}&
         \textbf{84.1}& \cellcolor{lightblue} \underline{202.5}&
         \textbf{85.7}& \cellcolor{lightblue} \underline{77.7}& 
         \textbf{84.8}& \cellcolor{lightblue} \underline{355.2}& 
         \textbf{75.6}& \cellcolor{lightblue} 98.4&
         \textbf{77.2}& \cellcolor{lightblue} \underline{353.4}& 
         \textbf{80.2}& \cellcolor{lightblue}\underline{189.9}&
         \textbf{95.4}& \cellcolor{lightblue}\underline{316.4}&
         \textbf{83.2} & \cellcolor{lightblue}\underline{198.0}\\
        \rowcolor{lightgray}& $\Delta$ & 
           +19.8 &  -7.3 & 
           +12.4 &  -8.4 &  
           +24.4 &  +44.1 & 
           +20.4 &  +20.1 &  
           +19.5 &  +8.8 & 
           +17.0 &  +31.4 &  
           +13.7 &  -7.2 &  
           +16.7 &  +33.5 &  
           +17.9 &  +20.1 &  
           +30.6 &  +62.9 &  
           +19.7 & +18.2 \\
        \hline\hline
    \end{tabular}
    }
    \caption{\textbf{Comparison with multi-concept method.} We compare FL2T with Multi-concept customization \cite{Kumari2023MultiConcept} on the CIFC dataset \cite{cidm} across all metrics, CLIP Image Alignment (IA), CLIP Text Alignment (TA), Identity Matching Scores (IMS) \cite{liu2024metacloakpreventingunauthorizedsubjectdriven} and Fretchet Inception Distance (FID) \cite{fid}, and notice that FL2T outperforms Multi-concept Customization across all metrics. }
    \label{tab:multi_concept}
\end{table*}

\subsection{Comparing Attention and Concatenation Operations}
\label{subsec:attn}
\noindent The main contribution of this work revolves around ``positively'' exploiting the higher-order interactions between concepts. Towards this end, we analyze two operations that are commonly utilized to extract richer representations from a set of embeddings. Consider a set of embeddings, $S = \{X, Y, Z\} \in \mathbb{R}^d$. We assume that trivial components such as linear layers for scaling in attention and downsampling in concatenation are present.

\paragraph{Attention.} Attention computes the cosine similarity between $X$ and $Y$ and adds the projected component of $X$ on $Y$ to $X$ - 
\[
    \text{Attn}(X; (X, Y)) = X + (X\cdot Y^T)Y 
\]

\noindent To develop a deeper understanding, we consider three vectors and let $a_{XY}$ be the cosine similarity between vectors $X$ and $Y$. Then, after the first attention operation where $S_1 = \{X_1, Y_1, Z_1\}$ is the output :
\[
\begin{split}
    X_1 &= X + a_{XY}Y + a_{XZ}Z \\
    Y_1 &= a_{XY}X + Y + a_{YZ}Z \\
    Z_1 &= a_{XZ}X + a_{YZ}Y + Z
\end{split}
\]
Subsequently, extending this to a second attention layer provides us:
\[
{\tiny
\begin{aligned}
X_2 =\; &\big[1 + (a^2_{XY} + a^2_{XZ})(a^2_{XY} + a^2_{YZ} + a^2_{XZ} + 3) 
          + 6a_{XY}a_{YZ}a_{XZ}\big]X  \\[6pt]
&+ \big[a_{XY}(3a^2_{YZ}-2) 
        + (a_{XY} + a_{XZ}a_{YZ})(a^2_{XY} + a^2_{YZ} + a^2_{XZ} + 6)\big]Y  \\[6pt]
&+ \big[a_{XZ}(3a^2_{YZ}-2) 
        + (a_{XZ} + a_{XY}a_{YZ})(a^2_{XY} + a^2_{YZ} + a^2_{XZ} + 6)\big]Z
\end{aligned}
}
\]
The number of pairwise interactions between the vectors of $S$ has significantly increased in the second attention layer, thereby allowing the model to capture higher-level dependencies between the vectors. It is important to note that attention is a generalized form of weighted aggregation.

\paragraph{Concatenation.} On the other hand, concatenation captures non-linear interactions between each element of the vectors. This can introduce unintended noise and alter the original embedding negatively. The operation is defined below where $a_i$, $b_i$ and $c_i$ are polynomials for the $i$-th element.

\[
\text{Concat(X, Y, Z)} = \sum_{i=0}^{d-1}a_iX_i + b_iY_i + c_iZ_i
\]

\begin{table}[!h]
\centering
\resizebox{0.47\textwidth}{!}{
\begin{tabular}{c|c|c|c}
\hline
\hline
\multirow{3}{*}{\textbf{Operation}} & \textbf{Pairwise} & \multirow{3}{*}{\shortstack{\textbf{After M}\\ \textbf{Layers}}} & \multirow{3}{*}{\shortstack{\textbf{Permutation} \\ \textbf{Invariance}}} \\
& \textbf{Interactions} & & \\
& \textbf{(per layer)} & & \\
\hline
\rowcolor{lightblue} \textbf{Summation ($x+y$)}               & $nd$      & $Mnd$ & $\checkmark$ \\

\textbf{Concatenation ($[x;y]$) }        & $0$       & $0$ & $\checkmark$\\


\rowcolor{lightblue} \textbf{Attention}                       & $n^2 d$   & $Mn^2 d$ & $\checkmark$ \\
\hline
\hline
\end{tabular}}
\caption{\textbf{Pairwise interactions.}We compute the number of pairwise interactions for combining two vectors $x,y\in\mathbb{R}^d$ across $n$ embeddings. We assume the same $k,d$ for $M$ layers. The attention function showcases its superiority by showing the largest number pairwise interactions and is permutation invariant.}
\label{tab:interactions_compact}
\end{table}

\subsection{Bounding Model Drift under Unnormalized Attention Coefficients} \label{subsec:proof}

We quantify one-step model drift by the norm of the aggregated gradient. Lemma~\ref{lem:upper-bound} shows a universal upper bound (the same crude bound as uniform summation), and Theorem~\ref{thm:existence-reduced} gives a simple, explicit construction proving that unnormalized attention can \emph{strictly reduce} drift relative to uniform summation.

\paragraph{Setup and Assumptions}
Let $(\mathbb{R}^d,\langle\cdot,\cdot\rangle)$ be a real inner-product space with norm $\|v\|=\sqrt{\langle v,v\rangle}$. We have $N$ concepts with losses $\ell_i(\theta)$ and gradients
\[
m_i \;:=\; \nabla_\theta \ell_i(\theta)\in\mathbb{R}^d,\qquad i=1,\dots,N.
\]
We compare two aggregate gradients:
\[
M_{\mathrm{CIDM}} \;:=\; \sum_{i=1}^N m_i\]
\[
M_{\mathrm{FL2T}} \;:=\; \sum_{i=1}^N \lambda_i\,m_i,\;\;\lambda_i\in[-1,1].
\]
We impose \emph{no} normalization on $(\lambda_i)$ (in particular, $\sum_i\lambda_i$ need not equal $1$). One-step drift with learning rate $\eta>0$ is proportional to $\|M\|$, so we compare $\|M_{\mathrm{FL2T}}\|$ to $\|M_{\mathrm{CIDM}}\|$.

\begin{lemma}[Universal Upper Bound]
\label{lem:upper-bound}
For any coefficients $\lambda_i\in[-1,1]$,
\[
\|M_{\mathrm{FL2T}}\|
\;=\;
\Big\|\sum_{i=1}^N \lambda_i m_i\Big\|
\;\le\;
\sum_{i=1}^N |\lambda_i|\,\|m_i\|
\;\le\;
\sum_{i=1}^N \|m_i\|.
\]
\end{lemma}

\begin{proof}
By the triangle inequality and homogeneity of norms,
{\small
\[
\Big\|\sum_{i=1}^N \lambda_i m_i\Big\|
\;\le\;
\sum_{i=1}^N \|\lambda_i m_i\|
\;=\;
\sum_{i=1}^N |\lambda_i|\,\|m_i\|
\;\le\;
\sum_{i=1}^N \|m_i\|
\]}
because $|\lambda_i|\le 1$ for all $i$.
\end{proof}

\begin{theorem}[Existence of Reduced Drift]
\label{thm:existence-reduced}
If $M_{\mathrm{CIDM}}\neq 0$, then there exists $\boldsymbol{\lambda}\in[-1,1]^N$, not all equal to $1$, such that$\|M_{\mathrm{FL2T}}\| \;<\; \|M_{\mathrm{CIDM}}\|$.
\end{theorem}

\begin{proof}
Set $M_{\mathrm{CIDM}} = \sum_{k=1}^N m_k$. Then
\begin{align*}
\sum_{k=1}^N \langle M_{\mathrm{CIDM}},\,m_k\rangle
&= \Big\langle M_{\mathrm{CIDM}},\,\sum_{k=1}^N m_k\Big\rangle \\
&= \langle M_{\mathrm{CIDM}},\,M_{\mathrm{CIDM}}\rangle \\
&= \|M_{\mathrm{CIDM}}\|^2 \;>\;0.
\end{align*}
Consequently, there exists an index $k^\star$ for which
\[
\langle M_{\mathrm{CIDM}}, m_{k^\star} \rangle > 0.
\]
Fix such an index $k^\star$. For any $\varepsilon \in (0,1]$, define coefficients
\[
\lambda_{k^\star} = 1 - \varepsilon,
\qquad
\lambda_j = 1 \quad \text{for all } j \neq k^\star.
\]
These coefficients lie in $[-1,1]$ and are not all equal to~1, so the resulting vector
\[
M_{\mathrm{FL2T}}
= \sum_{i=1}^N \lambda_i m_i
= M_{\mathrm{CIDM}} - \varepsilon\, m_{k^\star}
\]
is feasible. Expanding the squared norm yields
\begin{align*}
\|M_{\mathrm{FL2T}}\|^2
&= \|M_{\mathrm{CIDM}} - \varepsilon m_{k^\star}\|^2 \\
&= \|M_{\mathrm{CIDM}}\|^2
   - 2\varepsilon\,\langle M_{\mathrm{CIDM}}, m_{k^\star} \rangle
   + \varepsilon^2 \|m_{k^\star}\|^2
\end{align*}
Since $\langle M_{\mathrm{CIDM}}, m_{k^\star} \rangle > 0$, the quadratic expression on the right-hand side is strictly smaller than $\|M_{\mathrm{CIDM}}\|^2$ whenever
\[
0 < \varepsilon 
< \min \left\{
1,\;
\frac{2\,\langle M_{\mathrm{CIDM}},\, m_{k^\star}\rangle}{\|m_{k^\star}\|^2}
\right\}.
\]
Thus $\|M_{\mathrm{FL2T}}\|^2 < \|M_{\mathrm{CIDM}}\|^2$, which implies
$\|M_{\mathrm{FL2T}}\| < \|M_{\mathrm{CIDM}}\|$. Hence, there exists a feasible coefficient vector $\boldsymbol{\lambda}$ strictly reducing the norm, completing the proof.
\end{proof}

\paragraph{Remark (degenerate case).}
If $M_{\mathrm{CIDM}}=0$, then $\|M_{\mathrm{CIDM}}\|=0$ is already minimal, so no strict decrease is possible; nonetheless, Lemma~\ref{lem:upper-bound} still holds.

\begin{figure*}[!h]
    \centering
    \includegraphics[width=\textwidth]{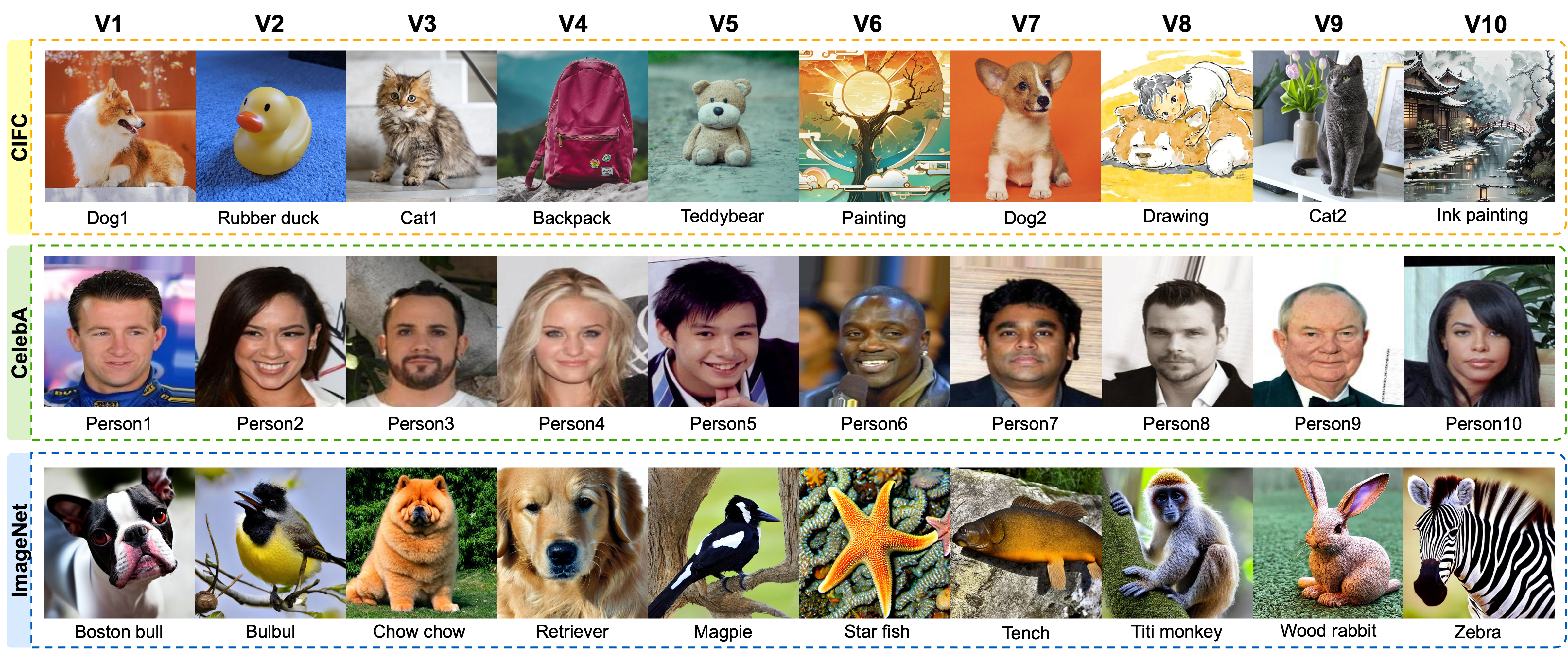}
    \caption{\textbf{Datasets used for the experiments.} We present an overview of the images used in our experiments for the CIDM \cite{cidm}, CelebA \cite{celeba} and ImageNet \cite{imagenet} datasets.}
    \label{fig:datasets}
\end{figure*}

\subsection{Background on Custom Diffusion Models.} 
\label{subsec:bg_cdms}

Latent diffusion models (LDMs) \cite{ldm1, ldm2} rely on conditional inputs, such as text prompts \cite{ldm_text1, ldm_text2} or images \cite{ldm_img1, ldm_img2}, to guide the generation of images. These models utilize an encoder $E(\cdot)$ and a decoder $D(\cdot)$ to facilitate image synthesis in the latent space. Custom diffusion models (CDMs) \cite{cdm1, cdm2, cdm3} extend LDMs by incorporating low-rank adaptation (LoRA) \cite{lora1, lora2} to fine-tune pretrained diffusion models \cite{finetune1, finetune2} for personalized concept learning.

Given a personalized image-text pair $(x, p)$, the encoder $E(\cdot)$ maps $x$ to a latent representation $z$, with $z_t$ denoting the noisy latent feature at timestep $t$ ($t = 1, \dots, T$). The text encoder $\Gamma(\cdot)$, such as a pretrained CLIP model \cite{clip}, maps the text prompt $p$ to a textual embedding $c = \Gamma(p)$. The objective for learning a personalized concept ${(x, p)}$ at timestep $t$ is defined as:

\begin{equation}\label{eq:1}
    L_{CDM} = \mathcal{E}_{z \sim E(x), c, \varepsilon \sim \mathcal{N}(0,I), t} \left[ \| \varepsilon - \varepsilon_{\theta'}(z_t | c, t) \|_2^2 \right]
\end{equation}

where $\varepsilon_{\theta'}(\cdot)$ represents the denoising UNet \cite{finetune2, ldm_text1} that gradually denoises $z_t$ by estimating the Gaussian noise $\varepsilon \sim \mathcal{N}(0, I)$. The parameter set $\theta'$ corresponds to  $\theta' = \theta_0 + \Delta\theta$, where $\theta_0 = \{W_l^0\}_{l=1}^{L}$ denotes the pretrained weights in LDMs, and $\Delta\theta = \{\Delta W_l\}_{l=1}^{L}$ corresponds to the LoRA-updated parameters \cite{mix_of_show, Kumari2023MultiConcept}. Here, $W_l^0, \Delta W_l \in \mathbb{R}^{a \times b}$ are the pretrained and low-rank weight matrices in the $l$-th transformer layer of $\theta'$, respectively, where $a$ and $b$ are matrix dimensions. Following \cite{dreambooth, motionDirector}, the low-rank update $\Delta W_l$ can be factorized as $\Delta W_l = A_l B_l$,
where $A_l \in \mathbb{R}^{a \times r}$ and $B_l \in \mathbb{R}^{r \times b}$ with rank $r \ll \min(a, b)$.

\begin{figure*}[!h]
    \centering
    \includegraphics[width=\textwidth]{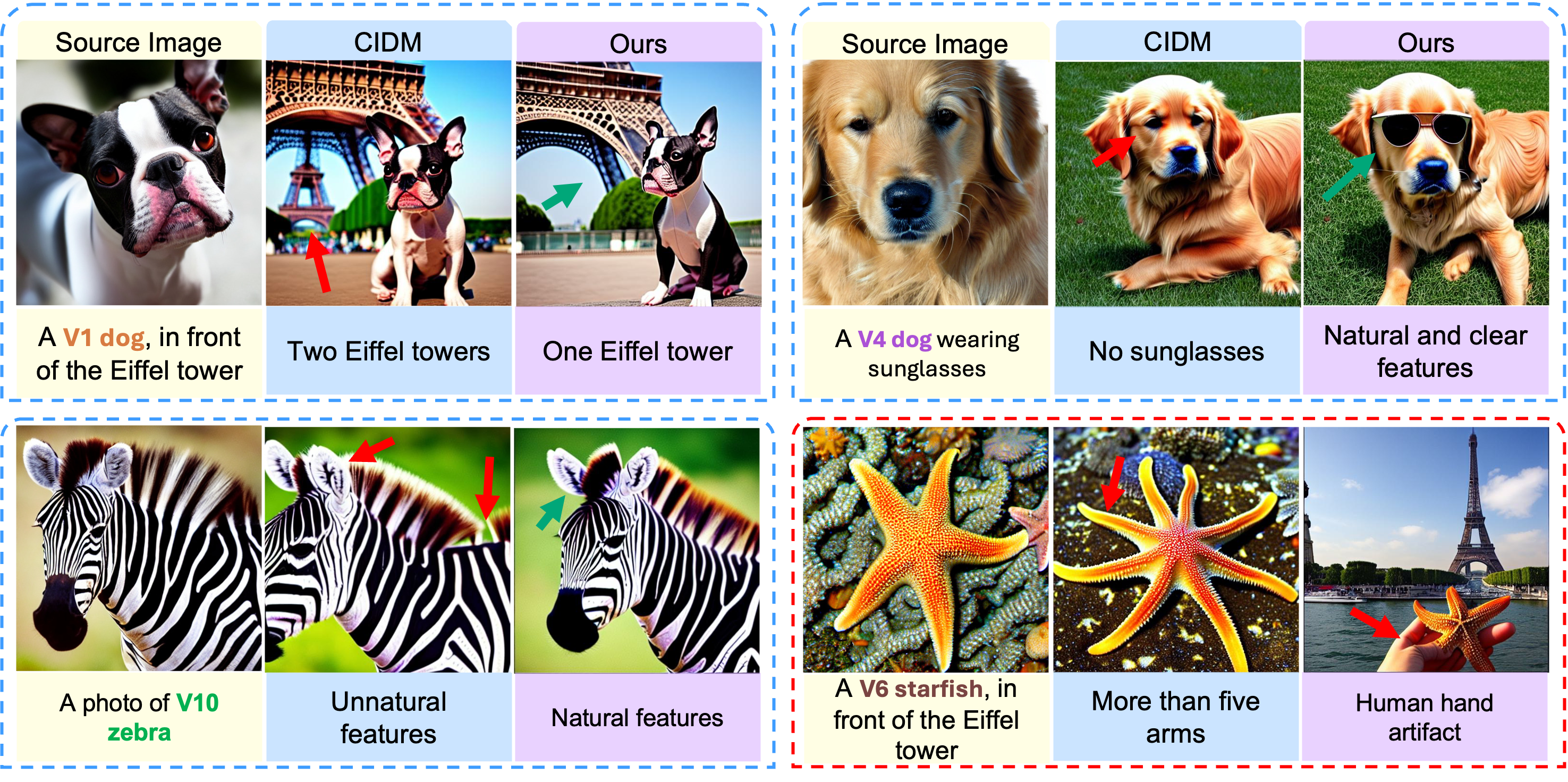}
    \caption{\textbf{Qualitative Analysis on the ImageNet dataset \cite{imagenet}.} We compare the synthesized images by CIDM \cite{cidm} and FL2T (Ours). The images are generated with a source image and an associated text prompt as input. Images with red and green arrows indicate regions of undesirable and desirable qualities, and their reasons are stated below each image. FL2T preseverves features and conforms to the text prompt better than CIDM. We have also shown a failure case of FL2T (bottom right in \textcolor{red}{red box}), where we observe a human hand artifact, but the identity of the starfish is preserved (only five arms).}
    \label{fig:qual_celeba}
\end{figure*}

\begin{figure*}[!h]
    \centering
    \includegraphics[width=\textwidth]{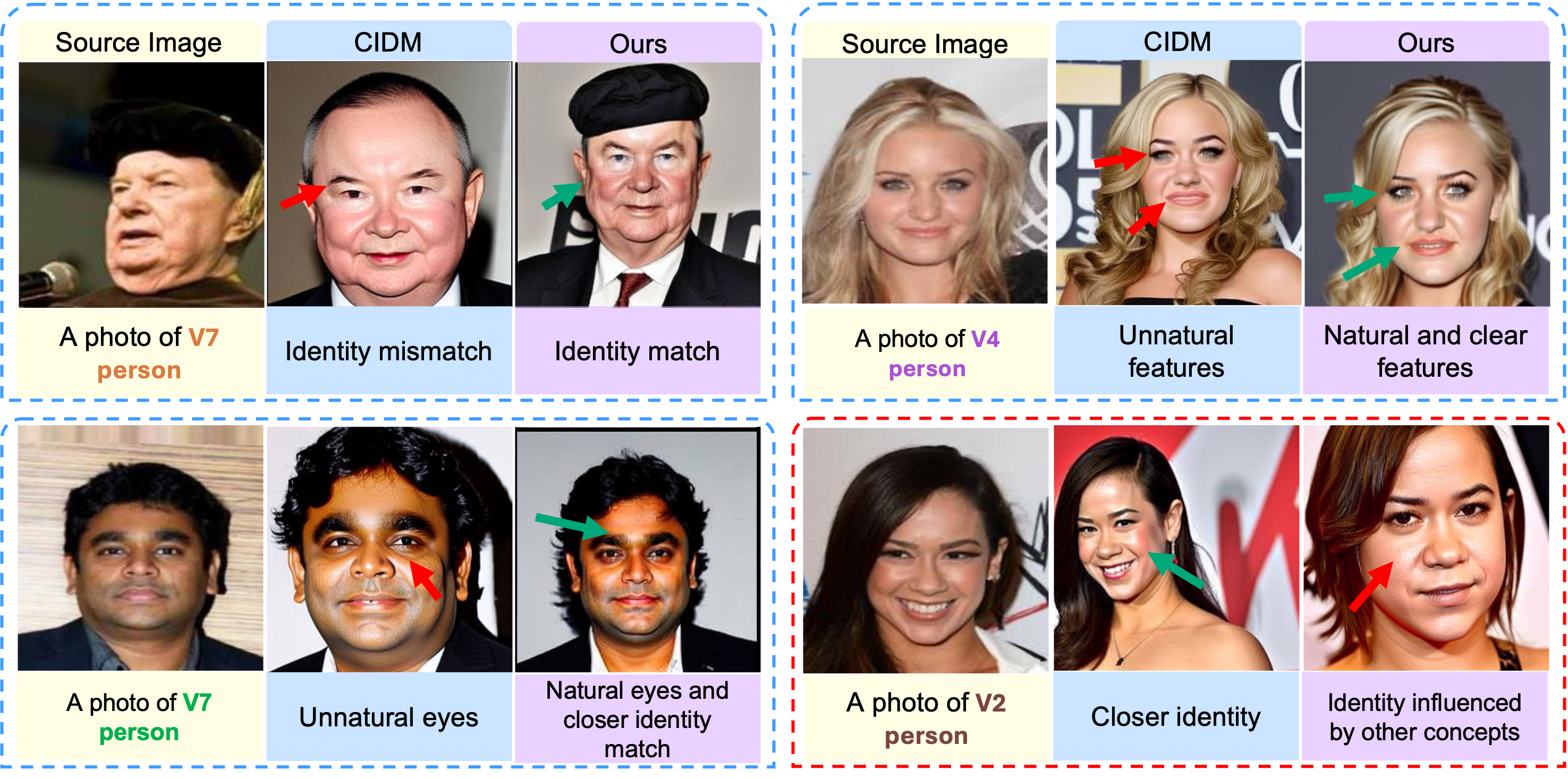}
    \caption{\textbf{Qualitative Analysis on the CelebA dataset \cite{celeba}. }We compare the synthesized images by CIDM \cite{cidm} and FL2T (Ours). The images are generated with a source image and an associated text prompt as input. Images with red and green arrows indicate regions of undesirable and desirable qualities, and their reasons are stated below each image. FL2T preseverves identity and features better than CIDM. We have also shown a failure case of FL2T (bottom right in \textcolor{red}{red box}), where the identity of V2 person was influenced by other concepts in the dataset.}
    \label{fig:qual_celeba}
\end{figure*}

\begin{figure*}[!h]
    \centering
    \includegraphics[width=\textwidth]{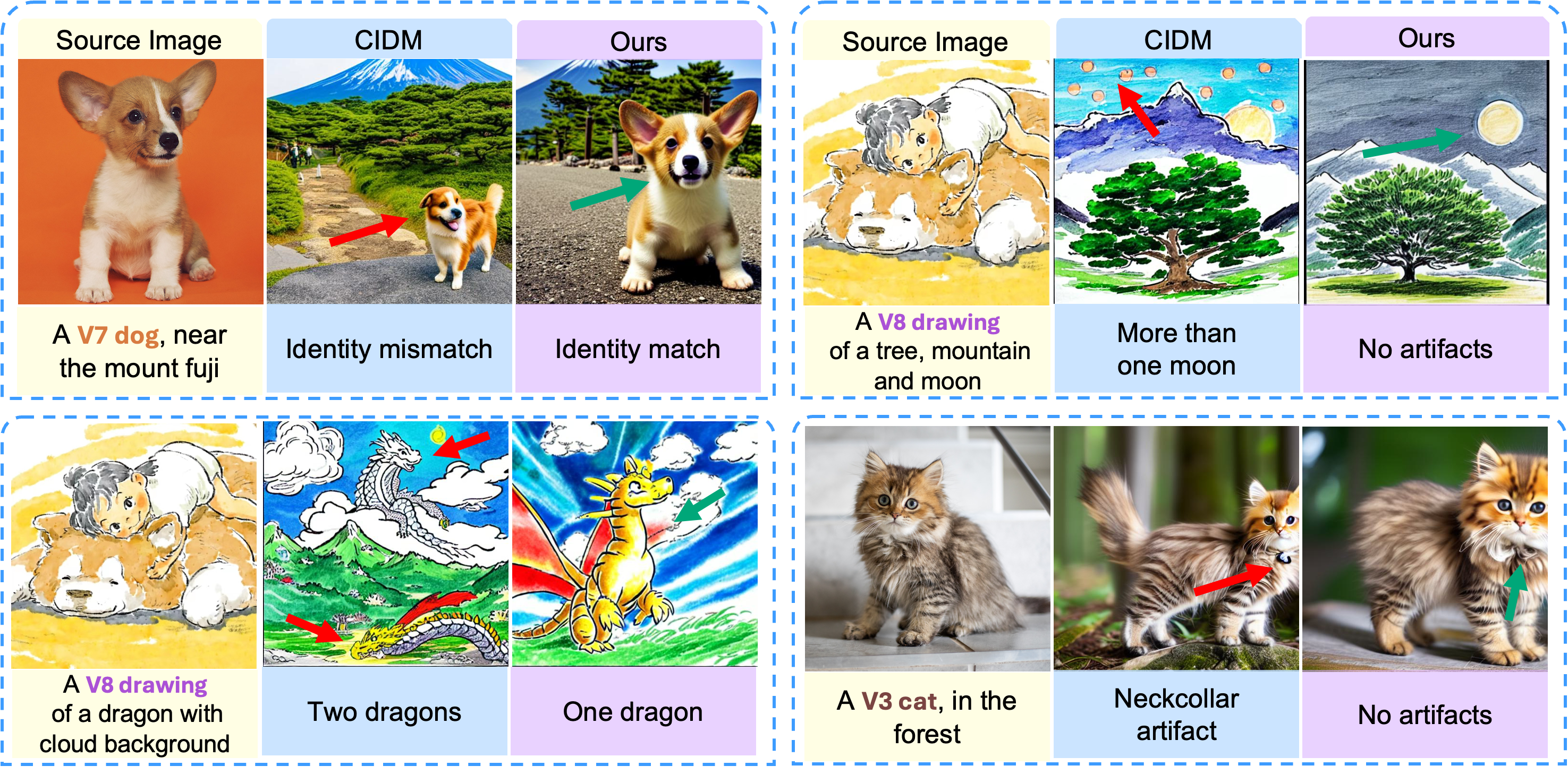}
    \caption{\textbf{More Qualitative Samples on the CIFC dataset \cite{cidm}.} We compare the synthesized images by CIDM \cite{cidm} and FL2T (Ours). The images are generated with a source image and an associated text prompt as input. Images with red and green arrows indicate regions of undesirable and desirable qualities, and their reasons are stated below each image. }
    \label{fig:qual_results2}
\end{figure*}


\end{document}